\documentclass{article}

\usepackage[dvipdfmx]{graphicx}
\usepackage{amsmath,amssymb,amscd,amsthm,bm}
\usepackage{color}
\usepackage[all]{xy}

\newtheorem{thm}{thm}[section]
\newtheorem{lem}[thm]{Lemma}
\newtheorem{prop}[thm]{Proposition}
\newtheorem{cor}[thm]{Corollary}
\newtheorem{exam}[thm]{Example}

\makeatletter 
\@tempcnta\z@
\loop\ifnum\@tempcnta<26
\advance\@tempcnta\@ne
\expandafter\edef\csname f\@Alph\@tempcnta\endcsname{\noexpand\mathfrak{\@Alph\@tempcnta}}
\expandafter\edef\csname l\@Alph\@tempcnta\endcsname{\noexpand\mathbb{\@Alph\@tempcnta}}
\expandafter\edef\csname c\@Alph\@tempcnta\endcsname{\noexpand\mathcal{\@Alph\@tempcnta}}
\expandafter\edef\csname b\@Alph\@tempcnta\endcsname{\noexpand\mathbf{\@Alph\@tempcnta}}
\repeat

\newcommand{\dd}{{\delta}}
\newcommand{\DD}{{\Delta}}
\newcommand{\ee}{{\varepsilon}}

\newcommand{\ra}{{\rightarrow}}

\newcommand{\hra}{{\hookrightarrow}}

\newcommand{\cech}{{\rm \check{C}ech}}

\newcommand{\Cech}{{\rm \check{\lC}ech}}
\newcommand{\Rips}{{\rm {\lR}ips}}
\newcommand{\Sub}{{\rm {\lS}ub}}
\newcommand{\Filt}{{\rm {\lF}ilt}}
\newcommand{\dist}{{\rm dist}}
\newcommand{\Pers}{{\rm Pers}}
\newcommand{\pers}{{\rm pers}}

\newcommand{\Lip}{{\rm Lip}}
\newcommand{\Amp}{{\rm Amp}}
\newcommand{\diam}{{\rm diam}}
\newcommand{\mesh}{{\rm mesh}}

\newcommand{\median}{{\rm median}}
\newcommand{\KFDR}{{\rm KFDR}}
\newcommand{\id}{{\rm id}}

\newcommand{\lmid}{ \ \middle| \ }

\providecommand{\abs}[1]{\left\lvert#1\right\rvert}
\providecommand{\norm}[1]{\left\lVert#1\right\rVert}
\providecommand{\pare}[1]{\left( #1 \right)}
\providecommand{\rl}[1]{\left\{ #1 \right\}}
\providecommand{\card}[1]{{\rm card}\pare{#1}}

\providecommand{\inn}[2]{\langle #1, #2 \rangle}
\providecommand{\Ek}[3]{E_{#1}(\mu^{#2}_{#3})}
\providecommand{\dk}[4]{\norm{E_{#1}(\mu^{#2}_{#3}) - E_{#1}(\mu^{#2}_{#4})}_{\cH_{#1}}}

\title{Kernel method for persistence diagrams via kernel embedding and weight factor}
\author{Genki Kusano \thanks{Tohoku University, genki.kusano.r5@dc.tohoku.ac.jp}
\and Kenji Fukumizu \thanks{The Institute of Statistical Mathematics, fukumizu@ism.ac.jp}
\and Yasuaki Hiraoka \thanks{Tohoku University, hiraoka@tohoku.ac.jp}}
\date{}
\begin{document}
\maketitle

\begin{abstract}

Topological data analysis is an emerging mathematical concept for characterizing shapes in multi-scale data. In this field, persistence diagrams are widely used as a descriptor of the input data, and can distinguish robust and noisy topological properties.  
Nowadays, it is highly desired to develop a statistical framework on persistence diagrams to deal with practical data. 
This paper proposes a kernel method on persistence diagrams. 
A theoretical contribution of our method is that the proposed kernel allows one to control the effect of persistence, and, if necessary, noisy topological properties can be discounted in data analysis. 
Furthermore, the method provides a fast approximation technique. The method is applied into several problems including practical data in physics, and the results show the advantage compared to the existing kernel method on persistence diagrams.
\end{abstract}

\section{Introduction}
\label{sec:intro}

Recent years have witnessed an increasing interest in utilizing methods of algebraic topology for statistical data analysis.
In terms of algebraic topology, conventional clustering methods are regarded as charactering $0$-dimensional topological features which mean connected components of data.
Furthermore, higher dimensional topological features also represent informative shape of data, such as rings ($1$-dimension) and cavities ($2$-dimension).
The research analyzing these topological features in data is called {\em topological data analysis} (TDA) \cite{Ca09}, which has been successfully applied to various areas including information science \cite{CIdSZ08,dSG07}, biology \cite{KZPSGP07,XW14}, brain science \cite{LCKKL11,PETCNHV14,SMISCR08}, biochemistry \cite{GHIKMN13}, material science \cite{HNHEMN16, NHHEN15, STRFH17}, and so on.
In many of these applications, data have complicated geometric structures, and thus it is important to extract informative topological features from the data.

A {\em persistent homology} \cite{ELZ02}, which is a key mathematical tool in TDA, extracts robust topological information from data, and it has a compact expression called a {\em persistence diagram}.
While it is applied to various problems such as the ones listed above, statistical or machine learning methods for analysis on persistence diagrams are still limited.
In TDA, analysts often elaborate only single persistence diagram and, in particular, methods for handling many persistence diagrams, which can contain randomness from the data, are at the beginning stage (see the end of this section for related works).
Hence, developing a framework of statistical data analysis on persistence diagrams is a significant issue for further success of TDA and, to this goal, this paper discusses kernel methods for persistence diagrams.

\subsection{Topological descriptor}
\label{subsec:persistent_homology}

In order to provide some intuitions for the persistent homology, let us consider a typical way of constructing persistent homology from data points in a Euclidean space, assuming that the point set lies on a submanifold.
The aim is to make inference on the topology of the underlying manifold from finite data points.
We consider the $r$-balls (balls with radius $r$) to recover the topology of the manifold, as popularly employed in constructing an $r$-neighbor graph in many manifold learning algorithms.
While it is expected that, with an appropriate choice of $r$, the $r$-ball model can represent the underlying topological structures of the manifold, it is also known that the result is sensitive to the choice of $r$.
If $r$ is too small (resp. large), the union of $r$-balls consists simply of the disjoint $r$-balls (resp. a contractible space).
Then, by considering not one specific $r$ but all $r$, the persistent homology gives robust topological features of the point set.
\begin{figure}[htbp]
\begin{center}
\includegraphics[width=0.9\hsize]{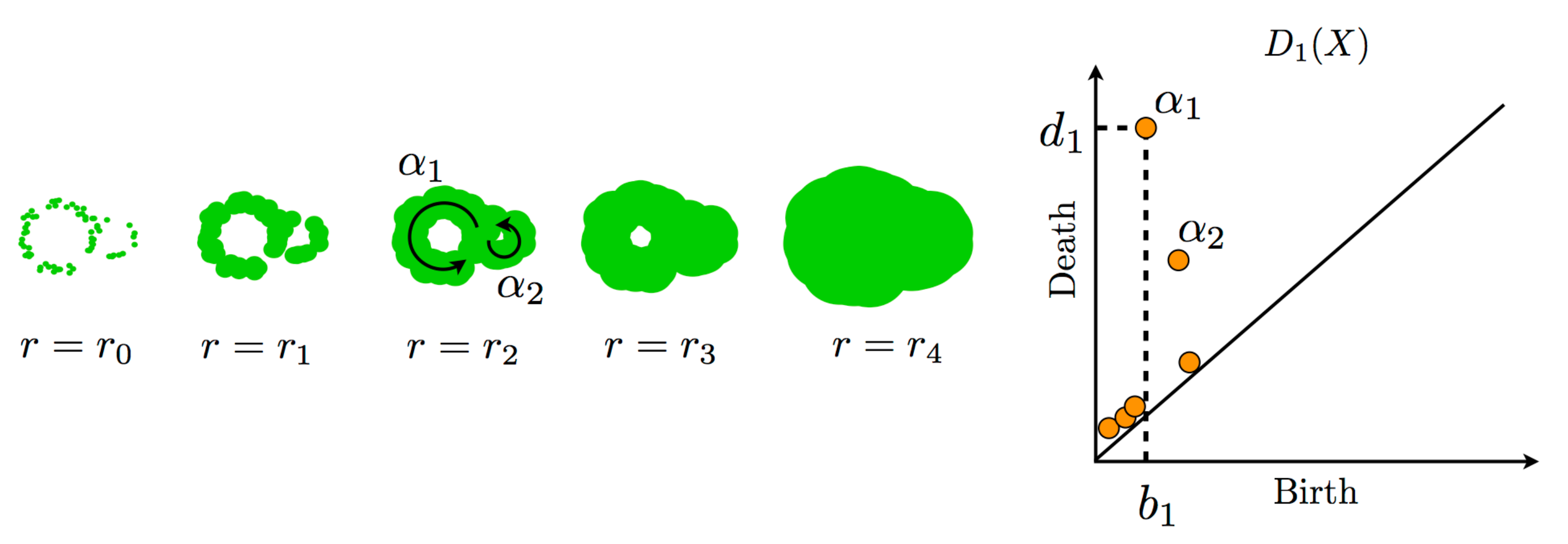}
\caption{Unions of $r$-balls at data points (left) and its $1$-st persistence diagram (right). The point $(b_{1},d_{1})$ in the persistence diagram represents the ring $\alpha_{1}$, which appears at $r=b_1$ and disappears at $r=d_1$. The noisy rings are plotted as the points close to the diagonal.}
\label{fig:filtration}
\end{center}
\end{figure}

As a useful representation of persistent homology, a persistence diagram is often used in topological data analysis.
The persistence diagram is given in the form of a multiset $D=\{(b_i,d_i) \in \lR^{2} \mid i\in I, \ b_i < d_i\}$ (Figure \ref{fig:filtration}).
Every point $(b_i,d_i)\in D$, called a {\em generator} of the persistent homology, represents a topological property (e.g., connected components, rings, cavities, etc.) which appears at $r=b_i$ and disappears at $r=d_i$ in the ball model.
Then, the {\em persistence} $d_i-b_i$ of the generator shows the robustness of the topological property under the radius parameter.
A generator with large persistence can be regarded as a reliable structure, while that with small persistence (points close to the diagonal) is likely to be a structure caused by noise.
In this way, persistence diagrams encode topological and geometric information of data points.
See Section \ref{sec:background} and Appendix \ref{sec:topology} for more information.

\subsection{Contribution}
\label{subsec:contribution}

Since a persistence diagram is a point set of variable size, it is not straightforward to apply standard methods of statistical data analysis, which typically assume vectorial data.
To vectorize persistence diagrams, we employ the framework of kernel embedding of (probability and more general) measures into reproducing kernel Hilbert spaces (RKHS).
This framework has recently been developed and leading various new methods for nonparametric inference \cite{MFSS17,SGSS07,SFG13}.
It is known \cite{SFL11} that, with an appropriate choice of kernels, a signed Radon measure can be uniquely represented by the Bochner integral of the feature vectors with respect to the measure.
Since a persistence diagram can be regarded as a sum of Dirac delta measures, it can be embedded into an RKHS by the Bochner integral.
Once such a vector representation is obtained, we can introduce any kernel methods for persistence diagrams systematically (see Figure \ref{fig:overview}). 
\begin{figure}[htbp]
\begin{center}
\includegraphics[width=0.95\hsize]{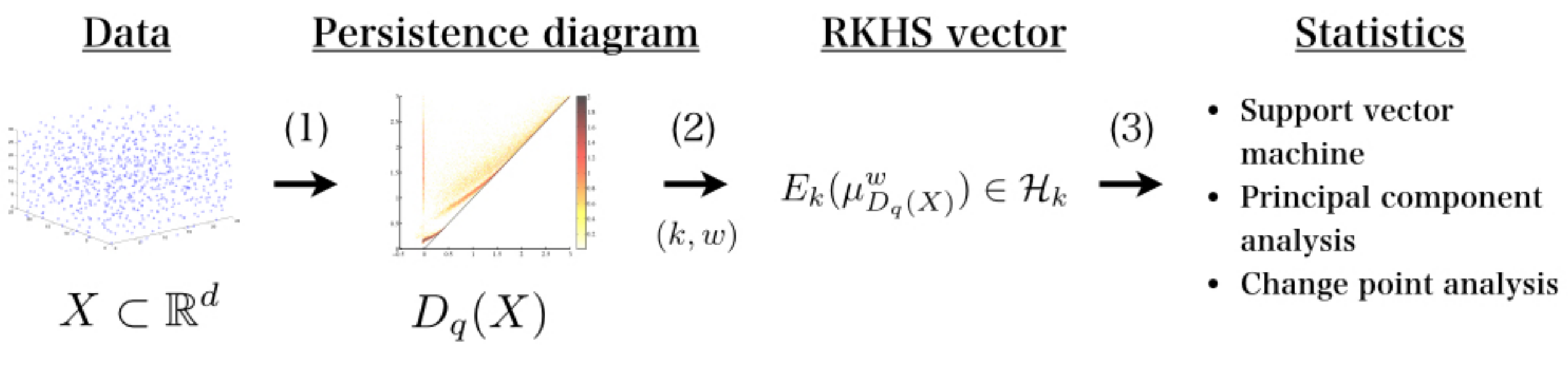}
\caption{
(1) A data set $X$ is transformed into a persistence diagram $D_q(X)$ (Section \ref{subsec:persistence_diagram}).
(2) The persistence diagram $D_q(X)$ is mapped to an RKHS vector $E_{k}(\mu_{D_{q}(X)}^{w})$, where $k$ is a positive definite kernel and $w$ is a weight function controlling the effect of persistence (Section \ref{subsec:vectorization}).
(3) Statistical methods are applied to those vector representations of  persistence diagrams (Section \ref{sec:experiment}).}
\label{fig:overview}
\end{center}
\end{figure}

Furthermore, since each generator in a persistence diagram is equipped with a persistence which indicates a robustness of the topological features, we will utilize it as a weight on the generator.
For embedding persistence diagrams in an RKHS, we propose a useful class of positive definite kernels, called {\em persistence weighted Gaussian kernel} (PWGK).
The advantages of this kernel are as follows:
(i) We can explicitly control the effect of persistence by a weight function, and hence discount the noisy generators appropriately in statistical analysis.
(ii) As a theoretical contribution, the distance defined by the RKHS norm for the PWGK satisfies the stability property, which ensures the continuity from data to the vector representation of the persistence diagram.
(iii) The PWGK allows efficient computation by using the random Fourier features \cite{RR07}, and thus it is applicable to persistence diagrams with a large number of generators.

We demonstrate the performance of the proposed kernel method with synthesized and real-world data, including granular systems (taken by X-ray Computed Tomography on granular experiments), oxide glasses (taken by molecular dynamics simulations) and protein datasets (taken by NMR and X-ray crystallography experiments).
We remark that these real-world problems have physical and biochemical significance in their own right, as detailed in Section \ref{sec:experiment}.

\subsection{Related works}
\label{subsec:related_work}

There are already some relevant works on statistical approaches to persistence diagrams.
Some studies discuss how to transform a persistence diagram to a vector \cite{AEKNPSCHMZ17, Bu15, CMWOXW15, COO15, RHBK15, RT16}.
In these methods, a transformed vector is typically expressed in a Euclidean space $\lR^{k}$ or a function space $L^{p}$, and simple and ad-hoc summary statistics like means and variances are used for data analysis such as principal component analysis (PCA) and support vector machines (SVMs).
In this paper, we will compare the performance among the PWGK, the persistence scale-space kernel \cite{RHBK15}, the persistence landscape \cite{Bu15}, the persistence image \cite{AEKNPSCHMZ17}, and the molecular topological fingerprint \cite{CMWOXW15} in several machine learning tasks.
Furthermore, we show that our vectorization is a generalization of the persistence scale-space kernel and the persistence image although the constructions are different.
We also remark that there are some works discussing statistical properties of persistence diagrams for random data points:
\cite{CGLM15} show convergence rates of persistence diagram estimation, and \cite{FLRWBS14} discuss confidence sets in a persistence diagram.
These works consider a different but important direction to the statistical methods for persistence diagrams.

The remaining of this paper is organized as follows:
In Section \ref{sec:background}, we review some basics on persistence diagrams and kernel embedding methods.
In Section \ref{sec:pdkernel},  the PWGK is proposed, and some theoretical and computational issues are discussed.
Section \ref{sec:experiment} shows experimental results and compares the proposed kernel method with other methods.

This paper is an extended version of our ICML paper \cite{KFH16}.
The difference from this conference version is as follows:
(i) Comparisons with other relevant methods, in particular, persistence landscapes and persistence images, have been added to this version.
(ii) New experimental results in comparison with other relevant methods.
(iii) Detailed proofs of the stability theorem has been added.

\section{Backgrounds}
\label{sec:background}

We review the concepts of persistence diagrams and kernel methods.  For readers who are not familiar with algebraic topology, we give a brief summary in Appendix \ref{sec:topology}.  See also \cite{Ha02} as an accessible introduction to algebraic topology.

\subsection{Persistence diagram}
\label{subsec:persistence_diagram}

In order to define a persistence diagram, we transform a data set $X$ into a filtration $\Filt(X)$ and compute its persistent homology $H_{q}(\Filt(X))$.
In this section, we will first introduce this mathematical framework of persistence diagrams. Then, by using a ball model filtration, we will intuitively explain geometrical meanings of persistence diagrams.
The ball model filtrations can be generalized toward two constructions using $\cech$ complexes and sub-level sets.
The former construction is useful for computations of persistence diagrams and the later is useful to discuss theoretical properties.

\subsubsection{Mathematical framework of persistence diagrams}

Let $K$ be a coefficient field of homology\footnote{In this setting, all homology are $K$-vector spaces.
You may simply consider the case $K=\lR$, but the theory is built with an arbitrary field.}.
Let $\Filt=\{F_a\mid a\in\mathbb{R}\}$ be a (right continuous) {\em filtration} of simplicial complexes (resp. topological spaces), i.e., $F_a$ is a subcomplex (resp. subspace) of $F_b$ for $a\leq b$ and $F_{a}=\bigcap_{a<b}F_{b}$.
For $a\leq b$, the $K$-linear map induced from the inclusion $F_a\hookrightarrow F_b$ is denoted by $\rho^b_a: H_q(F_a)\rightarrow H_q(F_b)$, where $H_{q}(F_{a})$ is the $q$-th homology of $F_{a}$.
The $q$-th {\em persistent homology} $H_q(\Filt)=(H_q(F_a),\rho^b_a)$ of $\Filt$ is defined by the family of homology $\{H_q(F_a)\mid a\in\mathbb{R}\}$ and the induced linear maps $\{ \rho^b_a \mid a\leq b \}$. 

A {\em homological critical value} of $H_q(\Filt)$ is the number $a\in\mathbb{R}$ such that the linear map $\rho^{a+\ee}_{a-\ee}: H_q(F_{a-\ee})\rightarrow H_q(F_{a+\ee})$ is not isomorphic for any sufficiently small $\ee>0$.
The persistent homology $H_q(\Filt)$ is called {\em tame} if $\dim H_q(F_a)<\infty$ for any $a\in\mathbb{R}$ and the number of homological critical values is finite. A tame persistent homology $H_q(\Filt)$ has a nice decomposition property:
\begin{thm}[\cite{ZC05}]\label{thm:decomposition}
A tame persistent homology can be uniquely expressed by
\begin{align}\label{eq:decom}
H_q(\Filt)\simeq\bigoplus_{i\in I} \lI[b_i,d_i],
\end{align}
where $\lI[b_i,d_i]=(U_a,\iota^b_a)$ consists of a family of $K$-vector spaces
\begin{align*}
U_a=\left\{\begin{array}{ll}
K,&b_i\leq a < d_i\\
0,&{\rm otherwise}
\end{array}\right.,
\end{align*}
and $\iota^b_a=\id_{K}$ for $b_i\leq a \leq b<d_i$.
\end{thm}

Each summand $\lI[b_i,d_i]$ means a topological feature in $\Filt$ that appears at $a=b_{i}$ and disappears at $a=d_{i}$.
The birth-death pair $x=(b_i,d_i)$ is called a {\em generator} of the persistent homology, and $\pers(x):=d_{i}-b_{i}$ a {\em persistence} of $x$.
We note that, when $\dim H_q(F_a) \neq 0$ for any $a<0$ (resp. for any $a>0$), the decomposition \eqref{eq:decom} should be understood in the sense that some $b_i$ takes the value $-\infty$ (resp. $d_i=\infty$), where $-\infty,\infty$ are the elements in the extended real $\overline{\mathbb{R}}=\mathbb{R}\cup\{-\infty,\infty\}$.
From the decomposition \eqref{eq:decom}, we define the {\em persistence diagram} of $\Filt$ as the multi-set\footnote{A {\em multi-set} is a set with multiplicity of each point.
We regard a persistence diagram as a multi-set, since several generators can have the same birth-death pairs.}
\begin{align*}
D_q(\Filt)=\rl{ (b_i,d_i)\in\overline{\mathbb{R}}^2 \lmid i\in I}.
\end{align*}

In this paper, we assume that all persistence diagrams have finite cardinality because a tame persistent homology defines a finite persistence diagram.
Moreover, we also assume that all birth-death pairs are bounded\footnote{This assumption will be justified in Section \ref{subsec:geometrical}.}, that is, all elements in a persistence diagram take neither $\infty$ nor $-\infty$.
Here, we define the (abstract) persistence diagram $D$ by a finite multi-set above the diagonal $\lR^{2}_{{\rm ad}}:=\{(b,d) \in \lR^{2} \mid b < d\}$.

\subsubsection{Ball model filtrations}
\label{subsec:geometrical}

The example used in Figure \ref{fig:filtration} can be expressed as follows.
Let $X = \{\bm{x}_1, \ldots, \bm{x}_n \}$ be a finite subset in a metric space $(M,d_{M})$ and $X_{a}:=\bigcup_{i=1}^{n} B(\bm{x}_{i};a)$ be a union of balls $B(\bm{x}_i;a)=\{ \bm{x} \in M \mid d_{M}(\bm{x}_i,\bm{x}) \leq a\}$ with radius $a \geq 0$.
For convenience, let $X_{a}:=\emptyset \ (a < 0)$.
Since $\lX=\{X_{a} \mid a \in \lR \}$ is a right-continuous filtration of topological spaces and $X$ is a finite set, $H_{q}(\lX)$ is tame and the persistence diagram $D_{q}(\lX)$ is well-defined.
For notational simplicity, the persistence diagram of this ball model filtration is denoted by $D_q(X)$.

We remark that, in this model, there is only one generator in $D_0(X)$ that does not disappear in the filtration; its lifetime is $\infty$.
From now on, we deal with $D_0(X)$ by removing this infinite lifetime generator\footnote{This is called the {\em reduced persistence diagram}.}.
Let $\diam(X)$ be the diameter of $X$ defined by $\max_{\bm{x}_{i},\bm{x}_{j} \in X} d_{M}(\bm{x}_{i},\bm{x}_{j})$.
Then, all generators appear after $a=0$ and disappear before $a=\diam(X)$ because $X_{\diam(X)}$ becomes a contractible space.
Thus, for any dimension $q$, all birth-death pairs of $D_{q}(X)$ have finite values.

\subsubsection{Geometric complexes}

We review some standard methods of constructing a filtration from finite sets in a metric space. See also \cite{CdSO14} for more details.

Let $(M,d_M)$ be a metric space and $X= \{\bm{x}_1, \ldots, \bm{x}_n \}$ be a finite subset in $M$.
For a fixed $a \geq 0$, we form a $q$-simplex $[\bm{x}_{i_0} \cdots \bm{x}_{i_q}]$ as a subset $\{ \bm{x}_{i_0}, \ldots, \bm{x}_{i_q} \}$ of $X$ whenever there exists $\bar{\bm{x}} \in M$ such that $d_M(\bm{x}_{i_j},\bar{\bm{x}}) \leq a$ for all $j=0,\ldots,q$, or equivalently, $\cap^{q}_{j = 0} B(\bm{x}_{i_{j}}; a) \neq \emptyset$.
The set of these simplices forms a simplicial complex, called the $\check{C}${\em ech complex} of $X$ with parameter $a$, denoted by $\cech(X;a)$.
For $a < 0$, we define $\cech(X;a)$ as an empty set.
Since there is a natural inclusion $\cech(X;a) \hra \cech(X;b)$ whenever $a \leq b$, $\Cech(X)=\left\{\cech(X;a) \ \middle| \ a \in \lR \right\}$ is a filtration.
When $M$ is a subspace of $\lR^{d}$, from the nerve lemma \cite{Ha02}, it is known that the topology of $\cech(X;a)$ is the same\footnote{Precisely, they are {\em homotopy equivalent}.} as $X_{a}$ (Figure \ref{fig:cech}), and hence $D_{q}(\Cech(X))=D_{q}(X)$.

The Rips complex (or Vietoris-Rips complex) is also often used in TDA and it gives different topology from the $\cech$ complex.
For a fixed $a \geq 0$, we form a $q$-simplex $[\bm{x}_{i_0} \cdots \bm{x}_{i_q}]$ as a subset $\left\{ \bm{x}_{i_0}, \ldots , \bm{x}_{i_q} \right\}$ of $X$ that satisfies $d_M(\bm{x}_{i_j},\bm{x}_{i_k}) \leq 2a$ for all $j,k=0,\ldots ,q$.
The set of these simplices forms a simplicial complex, called the {\em Rips complex} of $X$ with parameter $a$, denoted by ${\rm Rips}(X;a)$.
Similarly, the Rips complex also forms a filtration $\Rips(X)$.
In general, $D_{q}(\Rips(X))$ is not the same as $D_{q}(X)$ (see Figure \ref{fig:cech}).

\begin{figure}[htbp]
\begin{center}
\includegraphics[width=0.7\hsize]{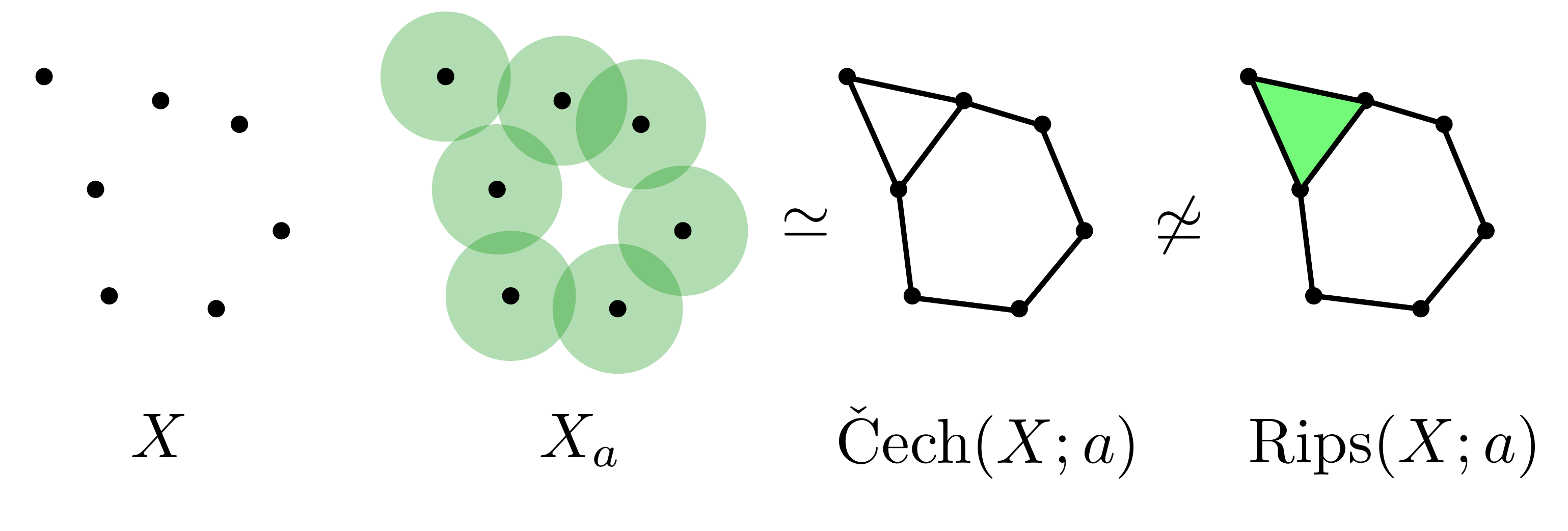}
\end{center}
\caption{A point set $X$, the union of balls $X_{a}$, the $\cech$ complex $\cech(X;a)$ and the Rips complex ${\rm Rips}(X;a)$. There are two rings in $X_{a}$ and $\cech(X;a)$. However, ${\rm Rips}(X;a)$ has only one ring because there is a $2$-simplex.}
\label{fig:cech}
\end{figure}

\subsubsection{sub-level sets}

Let $M$ be a topological space and $f:M \ra \lR$ be a continuous map.
Then, we define a {\em sub-level set} by ${\rm Sub}(f;a):=f^{-1}((-\infty,a])$ for $a \in \lR$ and its filtration by $\Sub(f):=\{{\rm Sub}(f;a) \mid a \in \lR \}$.
Here, $f:M \ra \lR$ is said to be {\em tame} if $H_{q}(\Sub(f))$ is tame.

For a finite set $X=\{\bm{x}_1, \ldots, \bm{x}_n \}$ in a metric space $(M, d_{M})$, we define the distance function $\dist_{X}:M \ra \lR$ by 
\[
\dist_{X}(\bm{x}):=\min_{\bm{x}_{i} \in X} d_{M} (\bm{x},\bm{x}_{i}).
\]
Then, we can see ${\rm Sub}(\dist_{X};a)=\bigcup_{x_{i} \in X} B(x_{i};a)$ and $D_{q}(\Sub(\dist_{X}))=D_{q}(X)$.
This means that the ball model is a special case of the sub-level set, and the $\cech$ complex and the sub-level set with the distance function $\dist_{X}$ give the same persistence diagram.
 
\subsection{Stability of persistence diagrams}
\label{sec:bottleneck_stability}

When we consider data analysis based on persistence diagrams, it is useful to introduce a distance measure among persistence diagrams for describing their relations.
In introducing a distance measure, it is desirable that, as a representation of data, the mapping from data to a persistence diagram is continuous with respect to the distance.
In many cases, data involve noise or stochasticity, and thus the persistence diagrams should be stable under perturbation of data.

The {\em bottleneck distance} $d_{{\rm B}}$ between two persistence diagrams $D$ and $E$ is defined by
\[
d_{{\rm B}}(D,E):=\inf_{\gamma} \sup_{x \in D \cup \DD} \norm{x-\gamma(x)}_{\infty},
\]
where $\DD:=\{(a,a) \mid a \in \lR \}$ is the diagonal set with infinite multiplicity and $\gamma$ ranges over all multi-bijections\footnote{A {\em multi-bijection} is a bijective map between two multi-sets counted with their multiplicity.} from $D \cup \DD$ to $E \cup \DD$.
Here, for $z=(z_1,z_2)\in\lR^2$, $\Vert z \Vert_\infty$ denotes $\max \{|z_1|,|z_2| \}$. 
We note that there always exists such a multi-bijection $\gamma$ because the cardinalities of $D \cup \DD$ and $E \cup \DD$ are equal by considering the diagonal set $\Delta$ with infinite multiplicity.
For sets $X$ and $Y$ in a metric space $(M, d_{M})$, let us recall the {\em Hausdorff distance} $d_{{\rm H}}$ given by
\begin{align*}
d_{{\rm H}}(X,Y):= \max \left\{\sup_{\bm{x} \in X} \inf_{\bm{y} \in Y} d_{M}(\bm{x},\bm{y}), \sup_{\bm{y} \in Y} \inf_{\bm{x} \in X} d_{M}(\bm{x},\bm{y}) \right\}.
\end{align*}
Then, the bottleneck distance satisfies the following stability property.
\begin{prop}[\cite{CdSO14,CEH07}]
\label{prop:point_stability}
Let $X$ and $Y$ be finite subsets in a metric space $(M,d_{M})$.
Then the persistence diagrams satisfy
\[
d_{{\rm B}}(D_{q}(X),D_{q}(Y)) \leq d_{{\rm H}}(X,Y).
\]
\end{prop}

Proposition \ref{prop:point_stability} provides a geometric intuition of the stability of persistence diagrams.
Assume that two point sets $X$ and $Y$ are close to each other with $\ee=d_{{\rm H}}(X,Y)$. If there is a generator $(b,d) \in D_{q}(Y)$, then we can find at least one generator in $X$ which is born in $(b-\ee,b+\ee)$ and dies in $(d-\ee,d+\ee)$ (see Figure \ref{fig:stability}).
Thus, the stability guarantees the similarity of two persistence diagrams, and hence we can infer the true topological features from the persistence diagrams given by noisy observation (see also \cite{FLRWBS14}).
\begin{figure}[htbp]
\begin{center}
\includegraphics[width=0.8\hsize]{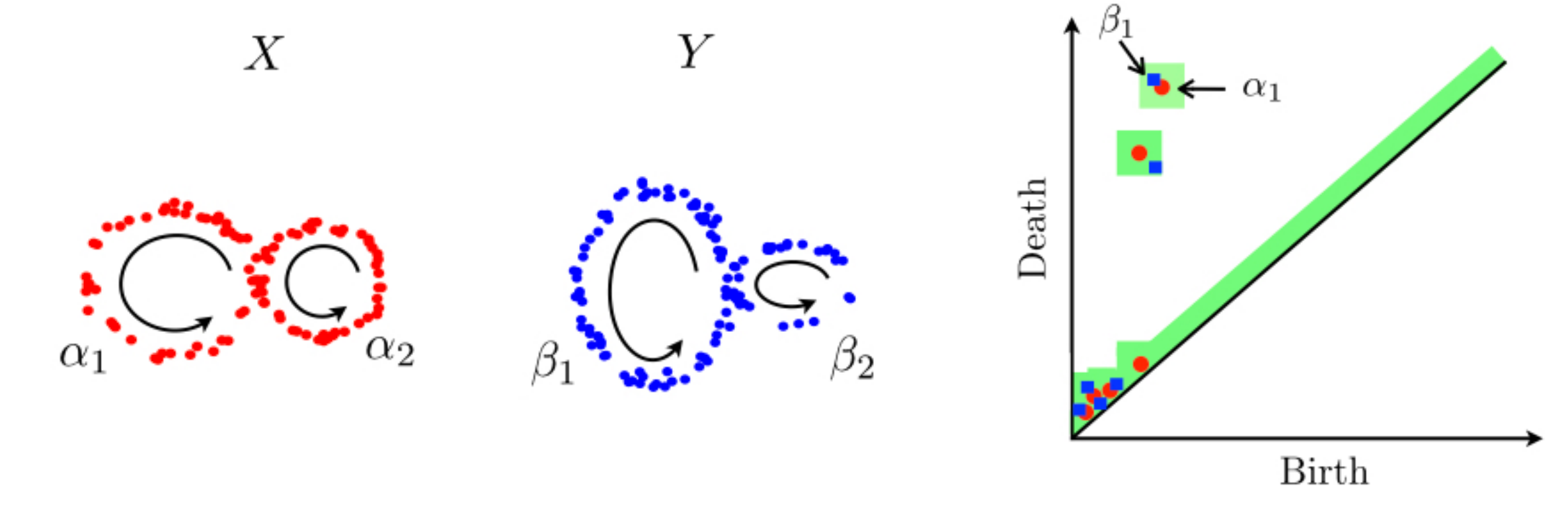}
\vspace{-3mm}
\caption{Two point sets $X$ and $Y$ (left) and their persistence diagrams (right).
The green region is an $\ee$-neighborhood of $D_{q}(Y)$ and all generators in $D_{q}(X)$ are in the $\ee$-neighborhood.}
\label{fig:stability}
\end{center}
\end{figure}

For $1 \leq p < \infty$, the {\em $p$-Wasserstein distance} $d_{{\rm W}_{p}}$, which is also used as a distance between two persistence diagrams $D$ and $E$, is defined by
\[
d_{ {\rm W}_{p}}(D,E)=\inf_{\gamma}\pare{\sum_{x \in D \cup \DD} \norm{x-\gamma(x)}^{p}_{\infty}}^{\frac{1}{p}},
\]
where $\gamma$ ranges over all multi-bijections from $D \cup \DD$ to $E \cup \DD$.
The $\infty$-Wasserstein distance $d_{{\rm W}_{\infty}}$ is defined by the bottleneck distance $d_{{\rm B}}$. Here, we define the {\em degree-$p$ total persistence} of $D$ by $\Pers_{p}(D):=\sum_{x \in D} \pers(x)^{p}$ for $1 \leq p < \infty$.

\begin{prop}[\cite{CEHM10}]
\label{prop:wasserstein_stability}
Let $1 \leq p' \leq p < \infty$, and $D$ and $E$ be persistence diagrams whose degree-$p'$ total persistences are bounded from above.
Then, 
\[
d_{ {\rm W}_{p}}(D, E) \leq \pare{\frac{\Pers_{p'}(D)+\Pers_{p'}(E)}{2}}^{\frac{1}{p}} d_{{\rm B}}(D,E)^{1-\frac{p'}{p}}.
\]
\end{prop}

For a persistence diagram $D$, its degree-$p$ total persistence is bounded from above by $\card{D} \times \max_{x \in D} \pers(x)^{p}$, where $\card{D}$ denotes the number of generators in $D$.
However, this bound may be weak because, in general, $\card{D}$ cannot be bounded from above.
In particular, if data set has noise, the persistence diagram often has many generators close to the diagonal.
Thus, it is desirable that the total persistence is bounded from above independently of $\card{D}$.
In the case of persistence diagrams obtained from a ball model filtration, we have the following upper bound (see Appendix \ref{sec:total} for the proof):
\begin{lem}
\label{lem:point_total}
Let $M$ be a triangulable compact subspace in $\lR^{d}$, $X$ be a finite subset of $M$, and $p>d$. 
Then, 
\[
\Pers_{p}(D_{q}(X)) \leq \frac{p}{p-d}C_{M}\diam(M)^{p-d},
\]
where $C_{M}$ is a constant depending only on $M$.
\end{lem}

\begin{cor}
\label{cor:point_wasserstein}
Let $M$ be a triangulable compact subspace in $\lR^{d}$, $X,Y$ be finite subsets of $M$, and $p \geq p' > d$. Then 
\begin{align*}
d_{{\rm W}_{p}} (D_{q}(X), D_{q}(Y) )
&\leq  \pare{ \frac{p'}{p'-d}C_{M}\diam(M)^{p'-d} }^{\frac{1}{p}} d_{{\rm B}}(D_{q}(X),D_{q}(Y))^{1-\frac{p'}{p}} \\
&\leq  \pare{ \frac{p'}{p'-d}C_{M}\diam(M)^{p'-d} }^{\frac{1}{p}} d_{{\rm H}}(X,Y)^{1-\frac{p'}{p}}
\end{align*}
where $C_{M}$ is a constant depending only on $M$.
\end{cor}

\subsection{Kernel methods for representing signed measures}
\label{subsec:universal}

As a preliminary to our proposal of vector representation for persistence diagrams, we briefly summarize a method for embedding signed measures with a positive definite kernel.

Let $\Omega$ be a set and $k:\Omega \times \Omega \ra \lR$ be a {\em positive definite kernel} on $\Omega$, i.e., $k$ is symmetric, and for any number of points $x_{1},\ldots,x_{n}$ in $\Omega$, the Gram matrix $\pare{k(x_{i},x_{j})}_{i,j=1,\ldots,n}$ is nonnegative definite.
A popular example of positive definite kernel on $\lR^{d}$ is the Gaussian kernel $k_{{\rm G}}(x,y)=e^{-\frac{\norm{x-y}^{2}}{2 \sigma^{2}}} \ (\sigma>0)$, where $\norm{\cdot}$ is the Euclidean norm in $\lR^{d}$.
From Moore-Aronszajn theorem, it is also known that every positive definite kernel $k$ on $\Omega$ is uniquely associated with a reproducing kernel Hilbert space $\cH_{k}$ (RKHS).

We use a positive definite kernel to represent persistence diagrams by following the idea of the kernel mean embedding of distributions \cite{MFSS17, SGSS07,SFL11}.
Let $\Omega$ be a locally compact Hausdorff space, $M_{{\rm b}}(\Omega)$ be the space of all finite signed Radon measures\footnote{A {\em Radon measure} $\mu$ on $\Omega$ is a Borel measure on $\Omega$ satisfying
(i) $\mu(C) < \infty$ for any compact subset $C \subset \Omega$, and 
(ii) $\mu(B)=\sup \{ \mu(C) \mid C \subset B, ~ C \mbox{:compact}\}$ for any $B$ in the Borel $\sigma$-algebra of $\Omega$.} on $\Omega$, and $k$ be a bounded measurable kernel on $\Omega$.
Since $\int \norm{k(\cdot,x)}_{\cH_{k}} d \mu(x)$ is finite, the integral $\int k(\cdot, x) d \mu(x)$ is well-defined as the Bochner integral \cite{DU77}.
Here, we define a mapping from $M_{{\rm b}}(\Omega)$ to $\cH_{k}$ by
\begin{equation}\label{eq:E_k}
E_{k}:M_{{\rm b}}(\Omega) \ra \cH_{k}, ~~ \mu \mapsto \int k(\cdot, x) d \mu(x).
\end{equation}

For a locally compact Hausdorff space $\Omega$, let $C_{0}(\Omega)$ denote the space of continuous functions vanishing at infinity\footnote{A function $f$ is said to {\em vanish at infinity} if for any $\ee >0$ there is a compact set $K \subset \Omega$ such that $\sup_{x \in K^{c}} |f(x)| \leq \ee$.}.
A kernel $k$ on $\Omega$ is said to be $C_{0}$-kernel if $k(\cdot,x) \in C_{0}(\Omega)$ for any $x \in \Omega$.
If $k$ is $C_{0}$-kernel, the associated RKHS $\cH_{k}$ is a subspace of $C_{0}(\Omega)$.
A $C_{0}$-kernel $k$ is called {\em $C_{0}$-universal} if $\cH_{k}$ is dense in $C_{0}(\Omega)$.
It is known that the Gaussian kernel $k_{{\rm G}}$ is $C_{0}$-universal on $\lR^{d}$ \cite{SFL11}.
When $k$ is $C_{0}$-universal, the vector $E_k(\mu)$ in the RKHS uniquely determines the finite signed measure $\mu$, and thus serves as a representation of $\mu$. We summarize the property as follows:
\begin{prop}[\cite{SFL11}]
\label{prop:C0_distance}
Let $\Omega$ be a locally compact Hausdorff space.
If $k$ is $C_{0}$-universal on $\Omega$, the mapping $E_{k}$ is injective. Thus,
\[
d_{k}(\mu,\nu)=\norm{E_{k}(\mu)-E_{k}(\nu)}_{\cH_{k}}
\]
defines a distance on $M_{{\rm b}}(\Omega)$.
\end{prop}

\section{Kernel methods for persistence diagrams}
\label{sec:pdkernel}
We propose a positive definite kernel for persistence diagrams, called the persistence weighted Gaussian kernel (PWGK), to embed the persistence diagrams into an RKHS.
This vectorization of persistence diagrams enables us to apply any kernel methods to persistence diagrams and explicitly control the effect of persistence.
We show the stability theorem with respect to the distance defined by the embedding and discuss the efficient and precise approximate computation of the PWGK.

\subsection{Vectorization of persistence diagrams}
\label{subsec:vectorization}

We propose a method for vectorizing persistence diagrams using the kernel embedding \eqref{eq:E_k} by regarding a persistence diagram as a discrete measure.
In vectorizing persistence diagrams, it is desirable to have flexibility to discount the effect of generators close to the diagonal, since they often tend to be caused by noise.
To this goal, we explain slightly different two ways of embeddings, which turn out to give the same inner product for two persistence diagrams. 

First, for a persistence diagram $D$, we introduce a measure $\mu^{w}_{D}:=\sum_{x \in D} w(x)\dd_{x}$ with a weight $w(x) >0$ for each generator $x \in D$ (Figure \ref{fig:weighted}), where $\dd_x$ is the Dirac delta measure at $x$.
By appropriately choosing $w(x)$, the measure $\mu^{w}_{D}$ can discount the effect of generators close to the diagonal.
A concrete choice of $w(x)$ will be discussed later.
\begin{figure}[htbp]
\begin{center}
\includegraphics[width=0.8\hsize]{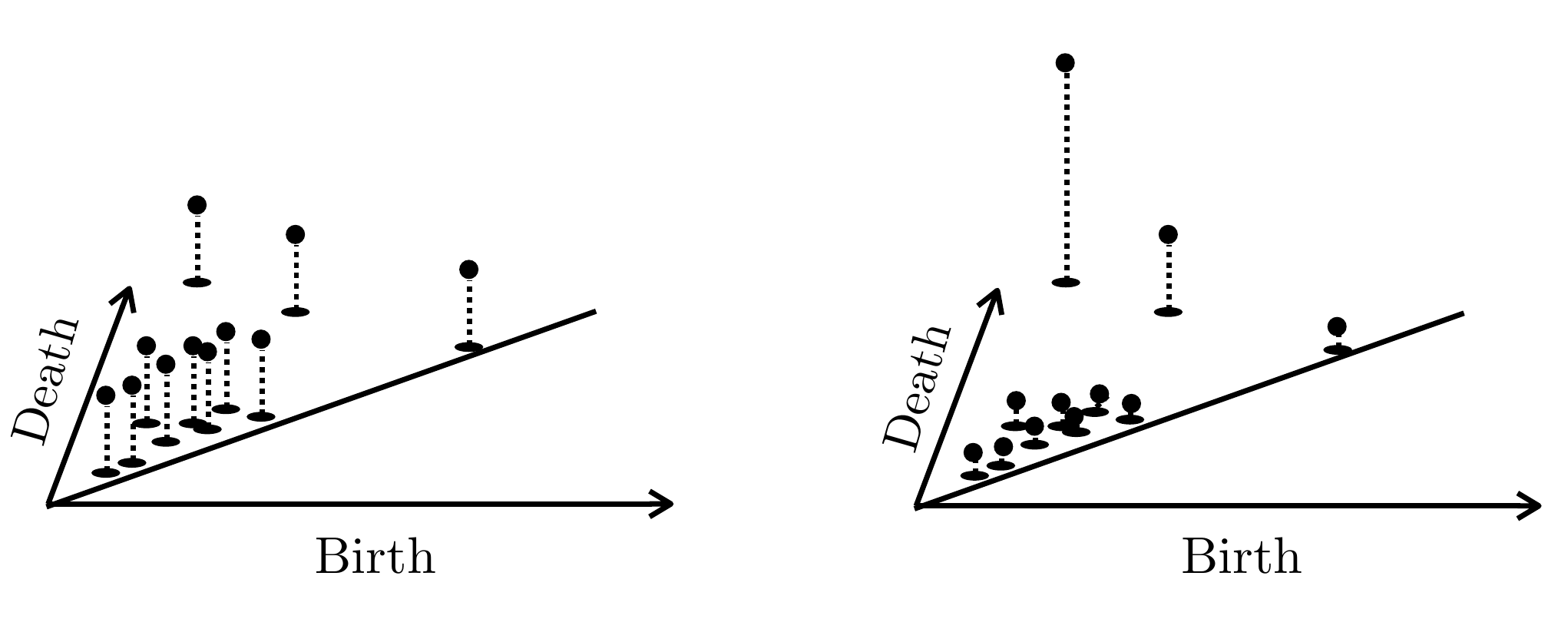}
\caption{Unweighted (left) and weighted (right) measures.}
\vspace{-3mm}
\label{fig:weighted}
\end{center}
\end{figure}

As discussed in Section \ref{subsec:universal}, given a $C_0$-universal kernel $k$ above the diagonal $\lR^{2}_{{\rm ad}}=\{(b,d) \in \lR^{2} \mid b < d\}$, the measure $\mu^{w}_{D}$ can be embedded as an element of the RKHS $\cH_{k}$ via 
\begin{equation}\label{E_k:embed}
\mu^{w}_{D} \mapsto E_{k}(\mu^{w}_{D}):=\sum_{x\in D}w(x)k(\cdot,x).
\end{equation}
From the injectivity in Proposition \ref{prop:C0_distance}, this mapping identifies a persistence diagram; in other words, it does not lose any information about persistence diagrams.
Hence, $E_k(\mu^w_D)\in\cH_k$ serves as a vector representation of the persistence diagram.

As the second construction, let
\[
k^w(x,y):=w(x)w(y)k(x,y)
\]
be the weighted kernel with the same weight function as above.
Then the mapping 
\begin{equation}\label{E_k^w:embed}
E_{k^w}: \mu_D\mapsto \sum_{x\in D}w(x)w(\cdot)k(\cdot,x)\in \cH_{k^w}
\end{equation}
also defines a vectorization of persistence diagrams.
The first construction may be more intuitive by directly weighting a measure, while the second one is also practically useful since all the parameter tuning is reduced to kernel choice.
We note that the inner products introduced by two RKHS vectors \eqref{E_k:embed} and \eqref{E_k^w:embed} are the same:
\[
\langle E_k(\mu_D^w), E_k(\mu_E^w)\rangle_{\cH_k} = \langle E_{k^w}(\mu_D), E_{k^w}(\mu_E)\rangle_{\cH_{k^w}}.
\]
In addition, these two RKHS vectors \eqref{E_k:embed} and \eqref{E_k^w:embed} are essentially equivalent, as seen from the next proposition:

\begin{prop}\label{prop:iso}
Let $k$ be $C_{0}$-universal on $\lR^{2}_{{\rm ad}}$ and $w$ be a positive function on $\lR^{2}_{{\rm ad}}$.
Then the following mapping
\[
\cH_k\to \cH_{k^w}, \quad f\mapsto wf
\]
defines an isomorphism between the RKHSs.  Under this isomorphism, $E_k(\mu_D^w)$ and $E_{k^w}(\mu_D)$ are identified. 
\end{prop}

\begin{proof}
Let $\tilde{\cH}:=\{wf:\lR^{2}_{{\rm ad}}\to\lR\mid f\in \cH_k\}$ and define its inner product by
\[
\langle wf,wg\rangle_{\tilde{\cH}}:=\langle f, g\rangle_{\cH_k}.
\]
Then, it is easy to see that $\tilde{\cH}$ is a Hilbert space and the mapping $f\mapsto wf$ gives an isomorphism between $\tilde{\cH}$ and $\cH_{k}$ of the Hilbert spaces.
In fact, we can show that $\tilde{\cH}$ is the same as $\cH_{k^w}$.
To see this, it is sufficient to check that $k^w$ is a reproducing kernel of $\tilde{\cH}$ from the uniqueness property of a reproducing kernel for an RKHS.
The reproducing property is proven from 
\[
\langle wf, k^w(\cdot,x)\rangle_{\tilde{\cH}}=\langle f,w(x)k(\cdot,x)\rangle_{\cH_k} = w(x)f(x) = (wf)(x).
\]
The second assertion is obvious from Equations \eqref{E_k:embed} and \eqref{E_k^w:embed}.
\end{proof}

\subsection{Stability with respect to the kernel embedding}
\label{subsec:stability}

Given a data set $X$, we compute the persistence diagram $D_{q}(X)$ and vectorize it as an element $E_{k}(\mu_{D_{q}(X)}^{w})$ of the RKHS.
Then, for practical applications, this map $X \mapsto E_{k}(\mu_{D_{q}(X)}^{w})$ should be stable with respect to perturbations to the data as discussed in Section \ref{sec:bottleneck_stability}. 

Let $D$ and $E$ be persistence diagrams and $\gamma:D \cup \DD \ra E \cup \DD$ be any multi-bijection. Here, we partition $D$ (resp. $\DD$) into $D_{1}$ and $D_{2}$ (resp. $\DD_{1}$ and $\DD_{2}$) such as 
\[
\gamma(D_{1}) \subset \lR^{2}_{{\rm ad}}, \ \gamma(D_{2}) \subset \DD, \ \gamma(\DD_{1}) \subset \lR^{2}_{{\rm ad}}, \ \gamma(\DD_{2}) \subset \DD.
\]
Then $D_{1} \cup \DD_{1}$ and $E$ are bijective under $\gamma$.
Now, let a weight function $w$ be zero on the diagonal $\DD$.
Then, the norm of the difference between RKHS vectors is calculated as follows:
\begin{align*}
&\dk{k}{w}{D}{E} \nonumber \\
&=\norm{ \sum_{x \in D} w(x)k(\cdot,x) -\sum_{y \in E} w(y) k(\cdot,y)}_{\cH_{k}} \nonumber \\
&=\norm{ \sum_{x \in D} w(x)k(\cdot,x) -\sum_{x \in D_{1} \cup \DD_{1}} w(\gamma(x)) k(\cdot,\gamma(x))}_{\cH_{k}} \nonumber \\
&=\norm{ \sum_{x \in D \cup \DD_{1}} \biggl(w(x)k(\cdot,x)-w(\gamma(x))k(\cdot,\gamma(x)) \biggr)   + \sum_{x \in D_{2}} w(\gamma(x))k(\cdot,\gamma(x))}_{\cH_{k}} \nonumber \\
& = \norm{ \sum_{x \in D \cup \DD_{1} } \biggl(w(x)k(\cdot,x)-w(\gamma(x))k(\cdot,\gamma(x))  \biggr)  }_{\cH_{k}} \nonumber \\
& = \norm{ \sum_{x \in D  } w(x) \biggl(k(\cdot,x)-k(\cdot,\gamma(x))  \biggr)  +  \sum_{x \in D \cup \DD_{1} } \biggl(w(x)-w(\gamma(x)) \biggr) k(\cdot,\gamma(x))    }_{\cH_{k}} \nonumber \\
&\leq \sum_{x \in D} w(x) \norm{k(\cdot,x)-k(\cdot,\gamma(x))}_{\cH_{k}}+ \sum_{x \in D \cup \DD_{1}} \abs{w(x)-w(\gamma(x))} \norm{k(\cdot,\gamma(x))}_{\cH_{k}}.
\end{align*}

In the sequel, we consider the Gaussian kernel $k_{{\rm G}}(x,y)=e^{-\frac{\norm{x-y}^{2}}{2 \sigma^{2}}} \ (\sigma>0)$ for a $C_{0}$-universal kernel.
Since $\norm{k_{{\rm G}}(\cdot,x)-k_{{\rm G}}(\cdot,y)}_{\cH_{G}} \leq \frac{\sqrt{2}}{\sigma} \norm{x-y}_{\infty}$ (Lemma \ref{lemm:Lip_k} in Appendix \ref{sec:stability}) and $\norm{k_{{\rm G}}(\cdot,x)}_{\cH_{k_{{\rm G}}}}=\sqrt{k_{{\rm G}}(x,x)} \equiv 1$ for any $x,y \in \lR^{2}$, we have 
\begin{align}
\dk{k_{{\rm G}}}{w}{D}{E} \leq  \frac{\sqrt{2}}{\sigma}\sum_{x \in D} w(x) \norm{x-\gamma(x)}_{\infty}+ \sum_{x \in D \cup \DD_{1}} \abs{w(x)-w(\gamma(x))}.  \label{eq:dk_halfway}
\end{align}

In this paper, we propose to use a weight function 
\[
w_{{\rm arc}}(x) = \arctan (C \pers(x)^{p}) ~~~ (C>0, \ p \in \lZ_{>0}).
\]
This is a bounded and increasing function of $\pers(x)$.
The corresponding positive definite kernel is
\begin{equation}
k_{{\rm PWG}}(x,y)=w_{{\rm arc}}(x)w_{{\rm arc}}(y)e^{-\frac{\norm{x-y}^{2}}{2 \sigma^{2}}}.
\end{equation}
We call it {\em persistence weighted Gaussian kernel} (PWGK).
This function $w_{{\rm arc}}$ gives a small (resp. large) weight on a noisy (resp. essential) generator. 
In addition, by appropriately adjusting the parameters $C$ and $p$ in $w_{{\rm arc}}$, we can control the effect of the persistence.  
Furthermore, we show that the PWGK has the following property: 

\begin{prop}
\label{prop:general_stability}
Let $p > 2$, and $D$ and $E$ be finite persistence diagrams whose degree-$(p-1)$ total persistence are bounded from above.
Then,
\begin{align*}
\dk{k_{{\rm G}}}{w_{{\rm arc}}}{D}{E} \leq L(D,E;C,p,\sigma) d_{{\rm B}}(D,E),
\end{align*}
where $L(D,E;C,p,\sigma)$ is a constant bounded from above by
\begin{align*}
\biggl\{ \frac{\sqrt{2}}{\sigma} \Pers_{p}(D)+ 2p \biggl( \Pers_{p-1}(D) + \Pers_{p-1}(E) \biggr)  \biggr\} C .
\end{align*}
\end{prop}

\begin{proof}
Let $d_{{\rm B}}(D,E) = \ee$ and $\gamma:D \cup \DD \ra E \cup \DD$ be a multi-bijection achieving the bottleneck distance, i.e., $\sup_{x \in D \cup \DD} \norm{x- \gamma(x)}_{\infty} = \ee$. We have already observed 
\begin{align*}
\dk{k_{{\rm G}}}{w_{{\rm arc}}}{D}{E} \leq  \frac{\sqrt{2}}{\sigma}\sum_{x \in D} w_{{\rm arc}}(x) \norm{x-\gamma(x)}_{\infty}+ \sum_{x \in D \cup \DD_{1}} \abs{w_{{\rm arc}}(x)-w_{{\rm arc}}(\gamma(x))}
\end{align*}
in Equation \eqref{eq:dk_halfway}. 
From Lemma \ref{lemm:w_continuous} in Appendix \ref{sec:stability}, the right-hand side of the above inequality is bounded from above by
\begin{align}
&  \frac{\sqrt{2}}{\sigma} \sum_{x \in D} w_{{\rm arc}}(x) \norm{x-\gamma(x)}_{\infty} + 2pC \sum_{x \in D \cup \DD_{1}}  \max \{\pers(x)^{p-1}, \pers(\gamma(x))^{p-1}\} \norm{x-\gamma(x)}_{\infty} \nonumber \\
& \leq  \frac{\sqrt{2}}{\sigma} C \ee \sum_{x \in D} \pers(x)^{p}  + 2pC \ee \sum_{x \in D \cup \DD_{1}} \max \{ \pers(x)^{p-1},    \pers(\gamma(x))^{p-1} \}   \label{eq:warc_bound} \\
& \leq \biggl\{ \frac{\sqrt{2}}{\sigma} \Pers_{p}(D)+ 2p \biggl( \Pers_{p-1}(D) + \Pers_{p-1}(\gamma(D \cup \DD_{1})) \biggr) \biggr\} C \ee \nonumber \\
&= \biggl\{ \frac{\sqrt{2}}{\sigma} \Pers_{p}(D)+ 2p  \biggl( \Pers_{p-1}(D) + \Pers_{p-1}(E) \biggr)  \biggr\} C \ee \label{eq:total_last}.
\end{align}
We have used the fact $w_{\rm arc}(x)\leq C \pers(x)^p$ in \eqref{eq:warc_bound} and $\Pers_{p-1}(\gamma(D \cup \DD_{1})) = \Pers_{p-1}(E)$ in \eqref{eq:total_last}.
Thus, if both degree-$(p-1)$ total persistences of $D$ and $E$ are bounded from above, since degree-$p$ total persistence of $D$ is also bounded from above from Proposition \ref{prop:persistence_inequality}, the coefficient of $\ee$ appearing in \eqref{eq:total_last} is bounded from above.
\end{proof}

The constant $L(D,E;C,p,\sigma)$ is dependent on $D$ and $E$, and hence we cannot say that the map $D \mapsto E_{k_{{\rm G}}}(\mu^{w_{\rm arc}}_{D})$ is continuous.
In the case of persistence diagrams obtained from ball model filtrations, from Lemma \ref{lem:point_total}, the PWGK satisfies the following stability property.
Recall that $D_{q}(X)$ denotes the persistence diagram to the ball model for $X$:

\begin{thm}
\label{thm:kernel_stability}
Let $M$ be a triangulable compact subspace in $\lR^{d}$, $X,Y \subset M$ be finite subsets, and $p>d+1$.
Then,
\[
\dk{k_{{\rm G}}}{w_{{\rm arc}}}{D_{q}(X)}{D_{q}(Y)} \leq L(M,d;C,p,\sigma) d_{{\rm B}}(D_{q}(X),D_{q}(Y)),
\]
where $L(M,d;C,p,\sigma)$ is a constant depending on $M,d,C,p,\sigma$.
\end{thm}

\begin{proof}
For any finite set $X \subset M$, from Lemma \ref{lem:point_total}, there exists a constant $C_{M}>0$ such that
\[
\Pers_{p}(D_{q}(X)) \leq \frac{p}{p-d} C_{M}\diam(M)^{p-d}.
\]
By replacing $D$ and $E$ with $D_{q}(X)$ and $D_{q}(Y)$ in \eqref{eq:total_last}, respectively, we have 
\begin{align*}
& \dk{k_{{\rm G}}}{w_{{\rm arc}}}{D_{q}(X)}{D_{q}(Y)} \\
& \leq \biggl\{ \frac{\sqrt{2}}{\sigma} \Pers_{p}(D_{q}(X))+ 2p \biggl( \Pers_{p-1}(D_{q}(X)) + \Pers_{p-1}(D_{q}(Y)) \biggr)  \biggr\} C d_{{\rm B}}(D_{q}(X),D_{q}(Y))  \\
& \leq \biggl( \frac{\sqrt{2}}{\sigma} \frac{p}{p-d} \diam(M) + \frac{4p(p-1)}{p-1-d}  \biggr) C_{M}\diam(M)^{p-1-d}C d_{{\rm B}}(D_{q}(X),D_{q}(Y))
\end{align*}
Then, $L(M,d;C,p,\sigma):=\biggl( \frac{\sqrt{2}}{\sigma} \frac{p}{p-d} \diam(M) + \frac{4p(p-1)}{p-1-d}  \biggr) C_{M}\diam(M)^{p-1-d}C$ is a constant independent of $X$ and $Y$.
\end{proof}

Let $\cP_{{\rm finite}}(M)$ be the set of finite subsets in a triangulable compact subspace $M \subset \lR^{d}$.
Since the constant $L(M,d;C,p,\sigma)$ is independent of $X$ and $Y$, Proposition \ref{prop:point_stability} and Theorem \ref{thm:kernel_stability} conclude that the map
\[
\cP_{{\rm finite}}(M) \ra \cH_{k_{{\rm G}}} , ~~ X \mapsto E_{k_{{\rm G}}}(\mu_{D_{q}(X)}^{w_{{\rm arc}}})
\]
is Lipschitz continuous.
Note again that this implies a desirable stability property of the PWGK with the ball model: small perturbation of data points in terms of the Hausdorff distance causes only small perturbation of the persistence diagrams in terms of the RKHS distance with the PWGK.

As the most relevant work to the PWGK, the persistence scale-space kernel (PSSK, \cite{RHBK15})\footnote{See Section \ref{subsubsec:pssk}.} provides another kernel method for vectorization of persistence diagrams and its stability result is shown with respect to $1$-Wasserstein distance.
However, to the best of our knowledge, $1$-Wasserstein stability with respect to the Hausdorff distance is not shown, that is, for point sets $X$ and $Y$, $d_{ {\rm W}_{1}}(D_{q}(X),D_{q}(Y))$ is not estimated by $d_{{\rm H}}(X,Y)$ such as Proposition \ref{prop:point_stability} or Corollary \ref{cor:point_wasserstein}.
Furthermore, it is shown \cite{RHBK15} that the PSSK does not satisfy the stability with respect to $p$-Wasserstein distance for $p>1$, including the bottleneck distance $d_{{\rm B}} = d_{{\rm W}_{\infty}}$, and hence it is not ensured that results obtained from the PSSK are stable under perturbation of data points in terms of the Hausdorff distance.
On the other hand, since the PWGK has the desirable stability (Theorem \ref{thm:kernel_stability}), it is one of the advantages of our method over the previous research\footnote{In fact, if we apply Theorem 3 in \cite{RHBK15} to the PWGK directly, it concludes that the PWGK also does not satisfy the bottleneck stability. However, by using Proposition \ref{prop:general_stability}, we can avoid this difficulty, and Theorem  \ref{thm:kernel_stability} holds. For more details, see Appendix \ref{sec:additive}.}.
In addition, by the similar way in \cite{RHBK15}, we show the stability with respect to $1$-Wasserstein distance for our kernel vectorization. 

\begin{prop}
\label{prop:pwgk_wasserstein}
Let $D$ and $E$ be persistence diagrams.
If a weight function $w$ is zero on the diagonal and there exist constants $c_{1},c_{2}>0$ such that 
\[
\abs{w(x)} \leq c_{1}, ~~ \abs{w(x)-w(y)} \leq c_{2} \norm{x-y}_{\infty}
\]
for any $x,y \in \lR^{2}$, then
\[
\dk{k_{{\rm G}}}{w}{D}{E} \leq \pare{ \frac{\sqrt{2}}{\sigma} c_{1}+c_{2}} d_{ {\rm W}_{1}}(D,E).
\]
\end{prop}

\begin{proof}
From Equation \eqref{eq:dk_halfway}, we have
\begin{align}
\dk{k_{{\rm G}}}{w}{D}{E} 
& \leq \frac{\sqrt{2}}{\sigma}\sum_{x \in D} w(x) \norm{x-\gamma(x)}_{\infty}+ \sum_{x \in D \cup \DD_{1}} \abs{w(x)-w(\gamma(x))} \label{eq:pwgk_inequality} \\
&\leq \frac{\sqrt{2}}{\sigma} c_{1} \sum_{x \in D} \norm{x-\gamma(x)}_{\infty} + c_{2} \sum_{x \in D \cup \DD_{1}} \norm{x-\gamma(x)}_{\infty} \nonumber 
\end{align}
Since this inequality holds for any multi-bijection $\gamma$,  we obtain the $1$-Wasserstein stability.
\end{proof}

The weight function $w_{{\rm arc}}$ is bounded from above by $\frac{\pi}{2}$, and for $p=1$, from Lemma \ref{lemm:w_continuous}, we have
\[
\abs{w_{{\rm arc}}(x)-w_{{\rm arc}}(y)} \leq 2C \norm{x-y}_{\infty} ~~~ (x,y \in \lR^{2}_{{\rm ad}}) .
\]
Therefore, from Proposition \ref{prop:pwgk_wasserstein}, the PWGK also have $1$-Wasserstein stability:
\begin{cor}
Let $p=1$, and $D$ and $E$ be persistence diagrams.
Then
\[
\dk{k_{{\rm G}}}{w_{{\rm arc}}}{D}{E} \leq \pare{ \frac{\pi}{\sqrt{2}\sigma} +2C} d_{ {\rm W}_{1}}(D,E).
\]
\end{cor}

For $p>1$, we have
\begin{align*}
\sum_{x \in D \cup \DD_{1}} \abs{w_{{\rm arc}}(x)-w_{{\rm arc}}(\gamma(x))}
& \leq 2pC \sum_{x \in D \cup \DD_{1}} \max \{ \pers(x)^{p-1}, \pers(\gamma(x))^{p-1} \} \norm{x-\gamma(x)}_{\infty} \\
& \leq 2pC \biggl( \Pers_{p-1}(D) + \Pers_{p-1}(E) \biggr) \sum_{x \in D \cup \DD_{1}}  \norm{x-\gamma(x)}_{\infty},
\end{align*}
from Lemma \ref{lemm:w_continuous} and, hence, from Equation \eqref{eq:pwgk_inequality}, we have
\[
\dk{k_{{\rm G}}}{w_{{\rm arc}}}{D}{E} \leq \biggl\{ \frac{\pi}{\sqrt{2}\sigma} +2pC \biggl( \Pers_{p-1}(D) + \Pers_{p-1}(E) \biggr) \biggr\} d_{ {\rm W}_{1}}(D,E).
\]
Although the above inequality does not directly imply the Lipschitz continuity of the PWGK for $p>1$ with respect to $1$-Wasserstein distance, combining with Lemma \ref{lem:point_total}, we have the following $1$-Wasserstein stability:
\begin{cor}
Let $M$ be a triangulable compact subspace in $\lR^{d}$, $X,Y \subset M$ be finite subsets, and $p>d+1$.
Then,
\[
\dk{k_{{\rm G}}}{w_{{\rm arc}}}{D_{q}(X)}{D_{q}(Y)} \leq \pare{ \frac{\pi}{\sqrt{2}\sigma} + \frac{4p(p-1)}{p-1-d} C_{M}\diam(M)^{p-1-d} C}  d_{ {\rm W}_{1}}(D_{q}(X),D_{q}(Y)),
\]
for some constant $C_{M} > 0$.
\end{cor}

\subsection{Kernel methods on RKHS}

Once persistence diagrams are represented as RKHS vectors, we can apply any kernel methods to those vectors by defining a kernel over the vector representation.
In a similar way to the standard vectors, the simplest choice is to consider the inner product as a linear kernel
\begin{align}
K_{{\rm L}}(D,E;k,w):= \inn{\Ek{k}{w}{D}}{\Ek{k}{w}{E}}_{\cH_{k}} =\sum_{x \in D} \sum_{y \in E} w(x)w(y)k(x,y)  \label{eq:linear_rkhs}
\end{align}
on the RKHS and we call it the {\em $(k,w)$-linear kernel}.  

If $k$ is a $C_{0}$-universal kernel and $w$ is strictly positive on $\lR^{2}_{{\rm ad}}$, from Proposition \ref{prop:C0_distance}, $\dk{k}{w}{D}{E}$ defines a distance on the persistence diagrams and it is computed as 
\[
\sqrt{K_{{\rm L}}(D,D;k,w)+K_{{\rm L}}(E,E;k,w)-2K_{{\rm L}}(D,E;k,w)}. 
\]
Then, we can also consider a nonlinear kernel
\begin{align}
K_{{\rm G}}(D, E; k,w) = \exp \pare{- \frac{1 }{ 2 \tau^{2} } \dk{k}{w}{D}{E}^{2} } \ (\tau >0)  \label{eq:gauss_rkhs}
\end{align}
on the RKHS and we call it the {\em $(k,w)$-Gaussian kernel}.

In this paper, if there is no confusion, we also refer to the $(k_{{\rm G}},w_{{\rm arc}})$-Gaussian kernel as the PWGK.
\cite{MFDS12} observed better performance with nonlinear kernels for some complex tasks and this is one of the reasons that we will use the Gaussian kernel on the RKHS.

\subsection{Computation of Gram matrix}
\label{subsec:calculation}

Let $\cD=\{D_{\ell} \mid \ell=1,\ldots,n\}$ be a collection of persistence diagrams.
In many practical applications, the number of generators in a persistence diagram can be large, while $n$ is often relatively small;
in Section \ref{subsec:glass}, for example, the number of generators is about 30000, while $n=80$.

If the persistence diagrams contain at most $m$ points, each element of the Gram matrix $(K_{{\rm G}}(D_{i},D_{j};k_{{\rm G}},w))_{i,j=1,\ldots,n}$ involves $O(m^2)$ evaluations of $e^{-\frac{\norm{x-y}^{2}}{2\sigma^{2}}}$, resulting the complexity $O(m^{2}n^{2})$ for obtaining the Gram matrix.
Hence, reducing computational cost with respect to $m$ is an important issue.

We solve this computational issue by using the random Fourier features \cite{RR07}.
To be more precise, let $z_{1},\ldots,z_{M_{{\rm rff}}}$ be random variables from the $2$-dimensional normal distribution $N( (0,0), \sigma^{-2} I)$ where $I$ is the identity matrix.  This method approximates $e^{-\frac{\norm{x-y}^{2}}{2\sigma^{2}}}$ by $\frac{1}{M_{{\rm rff}}}\sum_{a=1}^{M_{{\rm rff}}} e^{\sqrt{-1}z_{a}x} (e^{\sqrt{-1}z_{a}y})^{*}$, where $*$ denotes the complex conjugate.
Then,  $\sum_{x \in D_{i}} \sum_{y \in D_{j}} w(x)w(y)k_{{\rm G}}(x,y)$ is approximated by $\frac{1}{M_{{\rm rff}}}\sum_{a=1}^{M_{{\rm rff}}}B^{a}_{i} (B^{a}_{j})^{*}$, where $B^{a}_{\ell}=\sum_{x \in D_{\ell}} w(x) e^{\sqrt{-1}z_{a}x}$.
As a result, the computational complexity of the approximated Gram matrix is $O(mnM_{{\rm rff}}+n^2M_{{\rm rff}})$, which is linear to $m$.

We note that the approximation by the random Fourier features can be sensitive to the choice of $\sigma$.
If $\sigma$ is much smaller than $\norm{x-y}$, the relative error can be large.  
For example, in the case of $x=(1,2),y=(1,2.1)$ and $\sigma=0.01$, $e^{-\frac{\norm{x-y}^{2}}{2\sigma^{2}}}$ is about $10^{-22}$ while we observed the approximated value can be about $10^{-3}$ with $M_{{\rm rff}}=10^{3}$.
As a whole, these $m^{2}$ errors may cause a critical error to the statistical analysis.
Moreover, if $\sigma$ is largely deviated from the ensemble $\norm{x-y}$ for $x \in D_{i},y \in D_{j}$, then most values $e^{-\frac{\norm{x-y}^{2}}{2 \sigma^{2}}}$ become close to $0$ or $1$.

In order to obtain a good approximation and extract meaningful values, the choice of parameters is important.  For supervised learning such as SVM, we use the cross-validation (CV) approach.  For unsupervised case, we follow the heuristics proposed in \cite{GFTSSS07} and set 
\[
\sigma =\median \{ \sigma(D_{\ell}) \mid \ell=1,\ldots,n\}, \mbox{ where } \sigma(D)=\median \{ \norm{x_{i}-x_{j}} \mid x_{i},x_{j} \in D, \ i<j \},
\]
so that $\sigma$ takes close values to many $\norm{x-y}$.
For the parameter $C$, we also set 
\[
C=( \median \{\pers(D_{\ell}) \mid \ell=1,\ldots,n\} )^{-p},  \mbox{ where } \pers(D)=\median \{ \pers(x_{i}) \mid x_{i} \in D \}.
\]
Similarly, the parameter $\tau$ in the $(k,w)$-Gaussian kernel is defined by 
\begin{align}
\median \rl{ \dk{k}{w}{D_{i}}{D_{j}}  \lmid 1 \leq i < j \leq n}. \label{eq:tau}
\end{align}

\section{Experiments}
\label{sec:experiment}

In this section, we apply the kernel method of the PWGK to synthesized and real data, and compare the performance between the PWGK and other statistical methods of persistence diagrams. 
All persistence diagrams are obtained from the ball model filtrations and computed by CGAL \cite{DLY15} and PHAT \cite{BKRW14}. 
With respect to the dimension of persistence diagrams, we use $2$-dimensional persistence diagrams in Section \ref{subsec:granular} and $1$-dimensional ones in other parts.

\subsection{Comparison to previous works}
\label{subsec:comparison}

\subsubsection{Persistence scale-space kernel}
\label{subsubsec:pssk}
The most relevant work to our method is the one proposed by \cite{RHBK15}.
Inspired by the heat equation, they propose a positive definite kernel called {\em persistence scale-space kernel} (PSSK) $K_{{\rm PSS}}$ on the persistence diagrams:
\begin{align}
K_{{\rm PSS}}(D,E)=\inn{\Phi_{t}(D)}{\Phi_{t}(E)}_{L^{2}(\lR^{2})} =\frac{1}{8 \pi t} \sum_{x \in D} \sum_{y \in E}  e^{-\frac{ \norm{x-y}^{2} }{8 t}} - e^{-\frac{ \norm{x-\bar{y}}^{2} }{8 t}}, \label{eq:pssk}
\end{align}
where $\Phi_{t}(D)(x)=\frac{1}{4\pi t} \sum_{y \in D} e^{-\frac{\norm{x-y}^{2}}{4 t}} - e^{-\frac{\norm{x-\bar{y}}^{2}}{4 t}}$ and $\bar{y}:=(y^{2},y^{1})$ for $y=(y^{1},y^{2})$.
We note that $\Phi_{t}(D)$ also takes zero on the diagonal by subtracting the Gaussian kernels for $y$ and $\bar{y}$.  

In fact, we can verify that the $(k,w)$-linear kernel contains the PSSK.
Let $\tilde{D}:=D \cup D^{*}$ where $D^{*}=\{(d,b) \in \lR^{2} \mid (b,d) \in D\}$.
Then, $\Phi_{t}(D)$ can also be expressed as 
\begin{align*}
\Phi_{t}(D)=\frac{1}{4\pi t}\sum_{y \in \tilde{D}} w_{{\rm PSS}}(y) k_{{\rm G}}(\cdot,y) \ \mbox{ where }  \ 
w_{{\rm PSS}}(y)=
\begin{cases}
1, & y^{2}>y^{1} \\
0, & y \in \DD \\
-1, & y^{2} <y^{1}
\end{cases},
\end{align*}
which is equal to $\frac{1}{4\pi t} E_{k_{{\rm G}}}(\mu^{w_{{\rm PSS}}}_{\tilde{D}})$.
Furthermore, the inner product in $\cH_{k_{{\rm G}}}$ is
\begin{align}
K_{{\rm L}}(\tilde{D},\tilde{E};k_{{\rm G}},w_{{\rm PSS}})=\inn{E_{k_{{\rm G}}}(\mu^{w_{{\rm PSS}}}_{\tilde{D}})}{E_{k_{{\rm G}}}(\mu^{w_{{\rm PSS}}}_{\tilde{E}})}_{\cH_{k_{{\rm G}}}}=2\sum_{x \in D} \sum_{y \in E}  k_{{\rm G}}(x,y)-k_{{\rm G}}(x,\bar{y}). \label{eq:pssk_embedding}
\end{align}
By scaling the variance parameter $\sigma$ in the Gaussian kernel $k_{{\rm G}}$ and multiplying by an appropriate scalar, Equation \eqref{eq:pssk} is the same as Equation \eqref{eq:pssk_embedding}.
Thus, the PSSK can also be approximated by the random Fourier features. When we apply the random Fourier features for the PSSK, we set $\tilde{\sigma}=\median\{\sigma(\tilde{D}_{\ell}) \mid \ell =1,\cdots, n\}$ as before and $t=\frac{\tilde{\sigma}^{2}}{4}$.

While both methods discount noisy generators, the PWGK has the following advantages over the PSSK.
(i) The PWGK can control the effect of the persistence by $C$ and $p$ in $w_{{\rm arc}}$ independently of the bandwidth parameter $\sigma$ in the Gaussian factor, while in the PSSK only one parameter $t$ cannot adjust the global bandwidth and the effect of persistence simultaneously.
(ii) The PSSK does not satisfy the stability with respect to the bottleneck distance (see also remarks after Theorem \ref{thm:kernel_stability}).

\subsubsection{Persistence landscape}
\label{subsubsec:pl}
The {\em persistence landscape} \cite{Bu15} is a well-known approach in TDA for vectorization of persistence diagrams.
For a persistence diagram $D$, the persistence landscape $\lambda_{D}$ is defined by
\[
\lambda_{D}(k,t) = k \mbox{-th largest value of } \min \{ t-b_{i},d_{i}-t\}_{+},
\]
where $c_{+}$ denotes $\max \{c,0\}$, and it is a vector in the Hilbert space $L^{2}(\lN \times \lR)$.
Here, we define a positive definite kernel of persistence landscapes as a linear kernel on $L^{2}(\lN \times \lR)$:
\begin{align}
K_{{\rm PL}}(D,E):=\inn{\lambda_{D}}{\lambda_{E}}_{L^{2}(\lN \times \lR)}=\int_{\lR} \sum_{k=1} \lambda_{D}(k,t)\lambda_{E}(k,t) dt. \label{eq:landscape}
\end{align}
Since a persistence landscape does not have any parameters, we do not need to consider the parameter tuning.
However, the integral computation is required and it causes much computational time.
Let $\cD=\{D_{\ell} \mid \ell=1,\ldots,n\}$ be a collection of persistence diagrams which contain at most $m$ points.
Since $\lambda_{D_{i}}(k,t) \equiv 0$ for any $k>m, ~ t \in \lR, ~ i=1,\cdots,n$, calculating $\{\lambda_{D_{i}}(k,t) \mid k \in \lZ_{\geq 0}\}$, which needs sorting, is in $O(m \log m)$ (see also \cite{BD17}).
For a fixed $t$, we can calculate $( \sum_{k=1} \lambda_{D_{i}}(k,t)\lambda_{D_{j}}(k,t) )_{i,j=1,\cdots, n}$ in $O(nm \log m + n^{2}m)$, and the Gram matrix $(K_{{\rm PL}}(D_{i},D_{j}))_{i,j = 1,\cdots,n}$ in $O(M_{{\rm int}} (nm \log m +n^{2}m))$, where $M_{{\rm int}}$ is the number of partitions in the integral interval.
Theoretically speaking, this implies that it takes more time to calculate the Gram matrix of $K_{{\rm PL}}$ than the PWGK and the PSSK by the random Fourier features.

\subsubsection{Persistence image}
\label{subsubsec:pi}
As a finite dimensional vector representation of a persistence diagram, a {\em persistence image} is proposed in \cite{AEKNPSCHMZ17}.
First, we prepare a differentiable probability density function $\phi_{x}:\lR^{2} \ra \lR$ with mean $x$ and a weight function $w:\lR^{2}_{{\rm ad}} \ra \lR$. For a persistence diagram $D$, the {\em corresponding persistence surface} is defined by
\begin{align}
\rho_{D}(z) := \sum_{x \in D} w(x)\phi_{x}(z). \label{eq:pi}
\end{align}
Then, for fixed points $a_{0} < \cdots < a_{M} ~ (a_{i} \in \lR)$, the {\em persistence image} ${\rm PI}(D)$ is defined by an $M \times M$ matrix whose $(i,j)$-element is assigned to the integral of $\rho_{D}$ over the pixel $P_{i,j}:=(a_{i-1},a_{i}] \times (a_{j-1},a_{j}]$, i.e., 
\[
{\rm PI}(D)_{i,j}: = \int _{ P_{i,j}} \rho_{D}(z) dz.
\]
Since the persistence image can be regarded as an $M^{2}$-dimensional vector, we define a vector ${\rm PIV}(D) \in \lR^{M^{2}}$ by
\[
{\rm PIV}(D)_{i+M (j -1)}: = {\rm PI}(D)_{i, j} ~ ,
\]
and, in this paper, call it the persistence image vector.

In \cite{AEKNPSCHMZ17}, they use the $2$-dimensional Gaussian distribution $\frac{1}{2\pi \sigma^{2}}k_{{\rm G}}(x,z)$ as $\phi_{x}(z)$ and a piecewise linear weighting function $w_{{\rm pers}}(x)$ defined by
\begin{align*}
w_{{\rm pers}}(x):=
\begin{cases}
0 & (\pers(x) < 0) \\
\frac{1}{L} \pers(x)& (0 \leq \pers(x) \leq L) \\
1 & (\pers(x)>L)
\end{cases} ~~,
\end{align*}
where $L$ is a parameter.
In this paper, for a collection of persistence diagrams $\cD=\{D_{\ell} \mid \ell=1,\ldots,n\}$, we set $L$ as 
\[
L=\max \{ L(D_{\ell}) \mid \ell=1,\cdots,n \}, \mbox{ where } L(D)=\max \{ d_{i} \mid (b_{i},d_{i}) \in D\}. 
\]
For points $a_{0} < \cdots < a_{M}$ of a pixel $P_{i,j}=(a_{i-1},a_{i}] \times (a_{j-1},a_{j}]$, we set $a_{M}=L$ and $a_{i}=\frac{i}{M}a_{M}$ for $0 \leq i \leq M$\footnote{Here, we set $a_{0}=0$ because all generators in the ball model filtrations are born after $b=0$.}.

Here, by choosing $\phi_{x}$ and $w$ in the proposed way, we define a positive definite kernel of persistence image vector as a linear kernel on $\lR^{M^{2}}$:
\begin{align}
K_{{\rm PI}}(D,E)&:=\inn{{\rm PIV}(D)}{{\rm PIV}(E)}_{\lR^{M^{2}}} \nonumber \\
&=\sum_{i,j=1}^{M}{\rm PI}(D)_{i,j}{\rm PI}(E)_{i,j} \nonumber \\
&=\frac{1}{(2\pi\sigma^{2})^{2}} \sum_{x \in D} \sum_{y \in E} w_{{\rm pers}}(x)w_{{\rm pers}}(y) \sum_{i,j=1}^{M} \int_{P_{i,j}} k_{{\rm G}}(x,z) dz \int_{P_{i,j}} k_{{\rm G}}(y,z) dz . \label{eq:inner_pi}
\end{align}

If we choose $\phi_{x}(z)$ as a (normalized) positive definite kernel $k(x,z)$, the corresponding persistence surface \eqref{eq:pi} is the same as the RKHS vector $E_{k}(\mu^{w}_{D})$\footnote{\cite{AEKNPSCHMZ17} use a persistence diagram in birth-persistence coordinates. That is, by a linear transformation $T(b,d)=(b,d-b)$, a persistence diagram $D$ is transformed into $T(D)$. In this paper, in order to compare with the persistence image and the PWGK, we use birth-death coordinates.}.
Thus, it may be expected that the persistence image and the PWGK show similar performance for data analysis.
However, there are several differences between the persistence image and the PWGK.
(i) The mapping from a persistence diagram to the persistence image is not injective due to the discretization by the integral, on the other hand, the injectivity of the RKHS vector $E_{k}(\mu^{w}_{D})$ is ensured in Proposition \ref{prop:C0_distance}.
(ii) It is also shown that the persistence image has a stability result with respect to $1$-Wasserstein distance, but it does not satisfy the bottleneck stability (Remark 1 in \cite{AEKNPSCHMZ17}) or the Haussdorff stability as noted after Theorem \ref{thm:kernel_stability}.
(iii) The computational complexity of a persistence image does not depend on the number of generators in a persistence diagram, but instead, it depends on the number of pixels. We can reduce the computational time of the persistence image by choosing a small mesh size $M$. However, as data in Section \ref{subsec:Synthesized}, some situations need a fine mesh (i.e., a large mesh size). Thus, we have to be careful with the choice of mesh size.

\subsection{Classification with synthesized data}
\label{subsec:Synthesized}
We compare the performance among the PWGK, the PSSK, the persistence landscape, and the persistence image for a simple binary classification task with SVMs.

\subsubsection{Synthesized data}  

In this experiment, we design data sets so that important generators close to the diagonal must be taken into account to solve the classification task.

Let $S^{1}(x,y,r,N)$
 be a set composed of $N$ points sampled with equal distance from a circle in $2$-dimensional Euclidean space with radius $r$ centered at $(x,y)$.
When we compute the persistence diagram of $S^{1}(x,y,r,N)$ for $N>3$, there always exists a generator whose birth time is approximately $\frac{\pi r}{N}$ (here we use $\sin \theta \approx \theta$ for small $\theta$) and death time is $r$ (Figure \ref{fig:birth-death}).
\begin{figure}[htbp]
\begin{center}
\includegraphics[width=0.6\hsize]{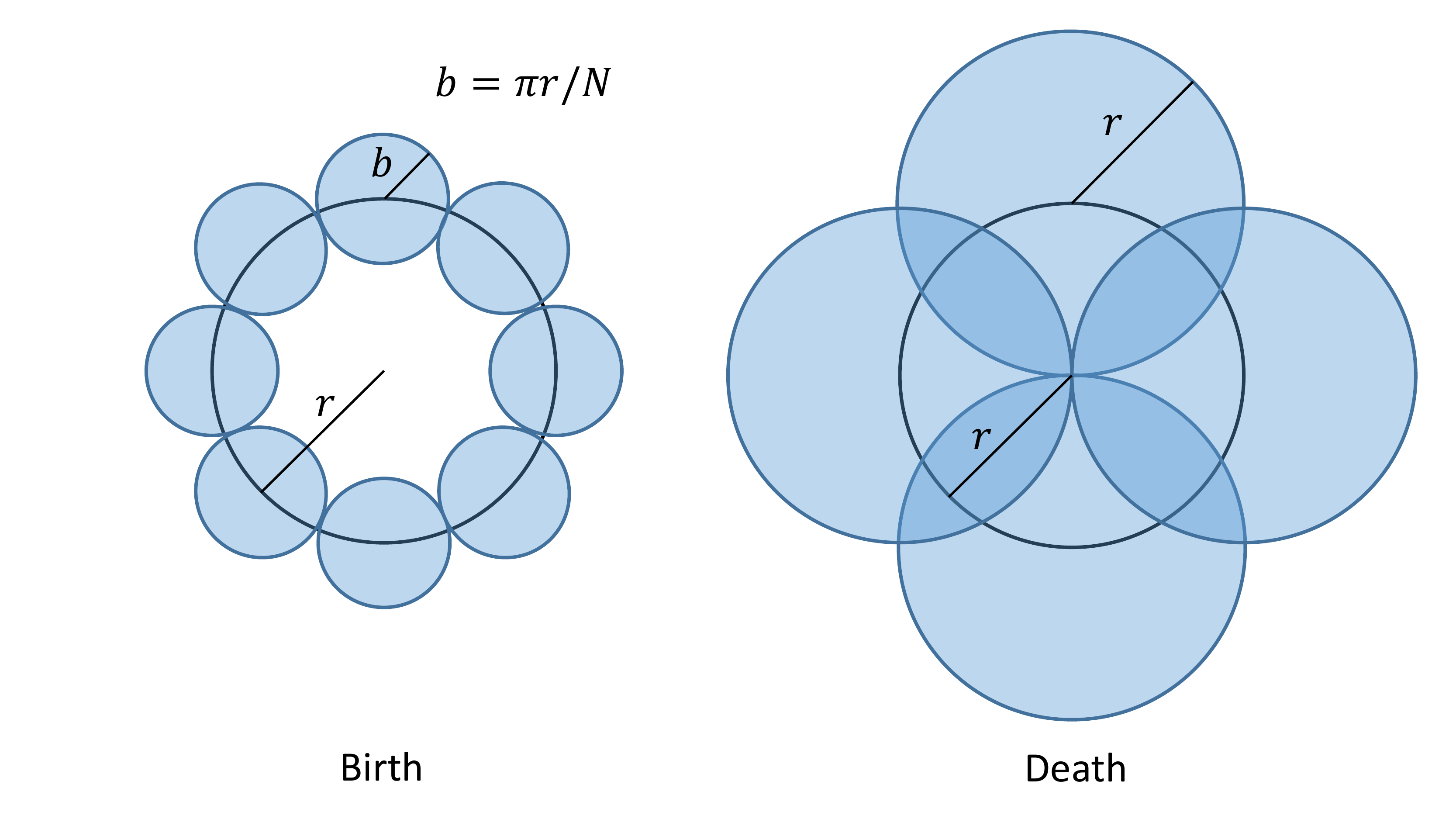}
\end{center}
\caption{Birth and death of the generator for $S^{1}(x,y,r,N)$.}
\label{fig:birth-death}
\end{figure}

In order to add randomness on $S^{1}(x,y,r,N)$, we extend it into $\lR^{3}$ and change $S^{1}(x,y,r,N)$ to $S_{z}^{1}(x,y,r,N)$ and $\tilde{S}_{z}^{1}(x,y,r,N)$ as follows:
\begin{align*}
S_{z}^{1}(x,y,r,N) &:=\{(z_{1},z_{2},z_{3}) \mid (z_{1},z_{2}) \in S^{1}(x,y,r,N), \ z_{3} \mbox{ is uniformly sampled from }[0, 0.01] \}\\
\tilde{S}_{z}^{1}(x,y,r,N) &:=S_{z}^{1}(x+W_{x}^{2},y+W_{y}^{2},r+W_{r}^{2}, \lceil N+2 W_{N} \rceil), 
\end{align*}
where $W_{x},W_{y} \sim N(0,2)$\footnote{$N(\mu,\sigma^{2})$ is the $1$-dimensional normal distribution with mean $\mu$ and variance $\sigma^{2}$.}, $W_{r},W_{N} \sim N(0,1)$ and $\lceil c \rceil$ is the smallest integer greater than or equal to $c$.
Then, we add $S_{2}:=S_{z}^{1}(x_{2},y_{2},r_{2},N_{2})$ to $S_{1}:=\tilde{S}_{z}^{1}(x_{1},y_{1},r_{1},N_{1})$ with probability $0.5$ and use it as the synthesized data.

In this paper, we choose parameters by
\begin{align*}
r_{1}&=1+8 W^{2} ~~ (W\sim N(0,1)), \\
x_{1}=y_{1}&=1.5 r_{1}, \\
N_{1}& ~:~ \mbox{a random integer with equal probability in } (\lceil \frac{\pi r}{2} \rceil, 4\pi r),
\end{align*}
and set $(x_{2},y_{2},r_{2},N_{2})$ as $(0,0,0.2,10)$ (Figure \ref{fig:synthesized}).
\begin{figure}[htbp]
\begin{center}
\includegraphics[width=0.6\hsize]{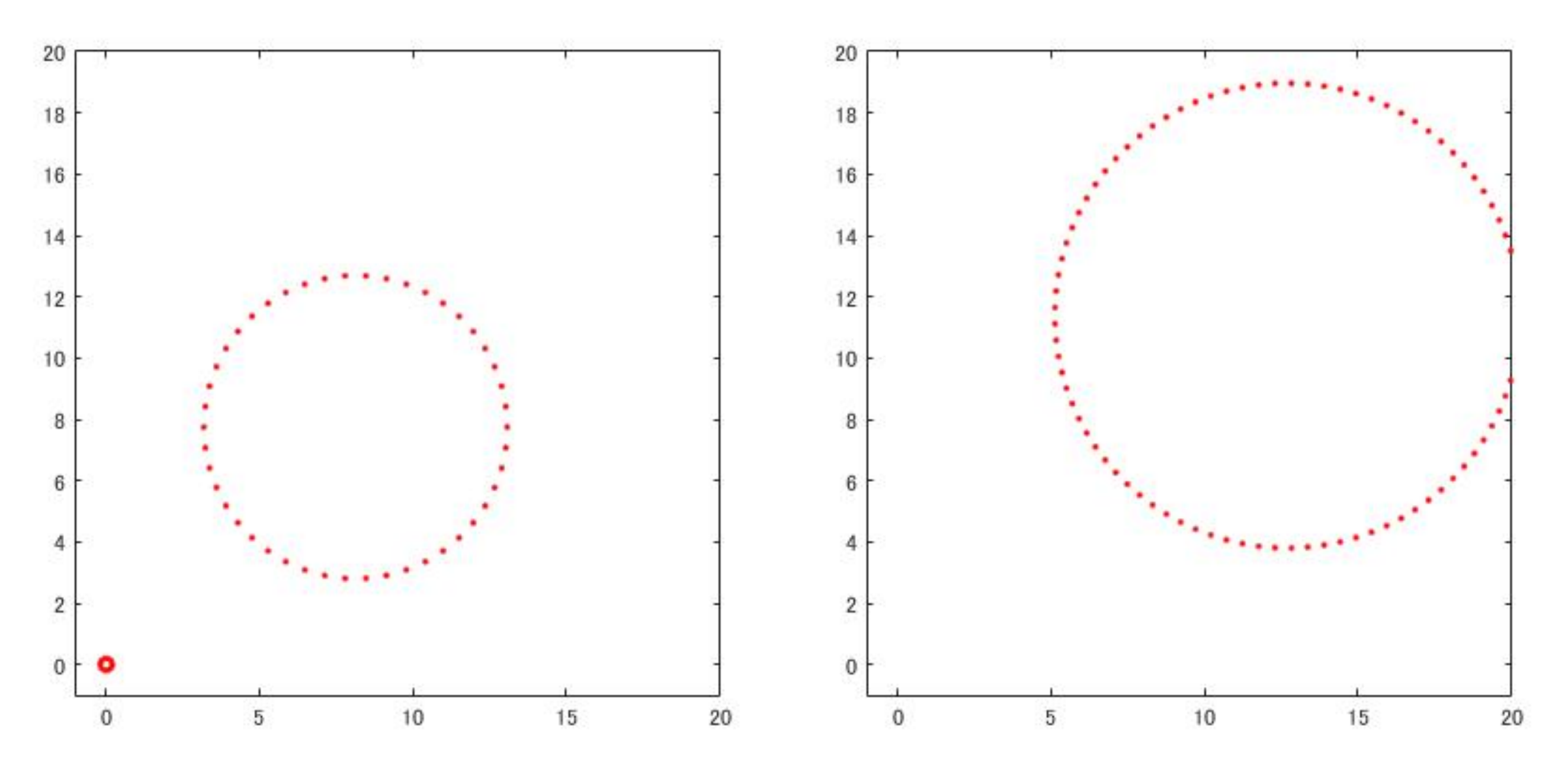}
\end{center}
\caption{Examples of synthesized data.  Left: $S_2$ exits. Right: $S_2$ does not exist. }
\label{fig:synthesized}
\end{figure}

For the binary classification, we introduce the following labels:
\begin{align*}
z_{0}=1 &~~  \mbox{if a generator for $S_{1}$ is born before $1$ and dies after $4$.} \\
z_{1}=1 &~~ \mbox{if $S_2$ exists.}
\end{align*}
The class label of the data set is then given by ${\bf XOR}(z_0,z_1)$.
By this construction, identifying $z_{0}$ requires relatively smooth function in the area of long lifetimes, while classifying the existing of $z_{1}$ needs delicate control of the resolution around the diagonal.

\subsubsection{SVM results} 

SVMs are trained from persistence diagrams given by 100 data sets, and evaluated with 100 independent test data sets.  
As a positive definite kernel $k$, we choose the Gaussian kernel $k_{{\rm G}}$ and the linear kernel $k_{{\rm L}}(x,y):=\inn{x}{y}_{\lR^{2}}$.
For a weight function $w$, we use the proposed function $w_{{\rm arc}}(x)=\arctan (C \pers(x)^{p})$, the piecewise linear weighting function $w_{{\rm pers}}(x)$ defined in Section \ref{subsubsec:pi}, and an unweighted function $w_{{\rm one}}(x) \equiv 1$.
The hyper-parameters $(\sigma, C)$ in the PWGK and $t$ in the PSSK are chosen by the 10-fold cross-validation, and the degree $p$ in $w_{{\rm arc}}(x)$ is set as $1, 5, 10$.
For $K_{{\rm PSS}}$ and $K_{{\rm PL}}$, while they originally consider only the inner product, we also apply the Gaussian kernels on RKHS following Equation \eqref{eq:gauss_rkhs}.
Since $K_{{\rm PI}}$ can be seen as a discretization of the $(k_{{\rm G}},w_{{\rm pers}})$-linear kernel, we also construct another kernel of persistence image by replacing $w_{{\rm pers}}$ with $w_{{\rm arc}}$, which is considered as a discretization of the PWGK.
In order to check whether the persistence image with $w_{{\rm arc}}$ is an appropriate discretization of the PWGK, we try several mesh size $M=20,50,100$.

\begin{table}[htbp]
\caption{Results of SVMs with the $(k,w)$-linear/Gaussian kernel, the PSSK, the persistence landscape, and the persistence image. Average classification rates ($\%$) and standard deviations for 100 test data sets are shown.}\label{table:Synth_results}
\centering
\begin{tabular}{ c  c | c | c }
\hline
\multicolumn{2}{c|}{} & Linear  & Gaussian \\ \hline
\multicolumn{2}{c|}{\textbf{PWGK}} &  	&    \\  
\textbf{kernel} & \textbf{weight} &  &  \\
			& $w_{{\rm arc}} \ (p=1)$ 	&  75.7 $\pm$ 2.31   		& 85.8 $\pm$ 5.19 (PWGK) \\
			& $w_{{\rm arc}} \ (p=5)$ 	&  75.8 $\pm$ 2.47 ($\triangle$)		& 85.6 $\pm$ 5.01 (PWGK, $\square$) \\
$k_{{\rm G}}$   & $w_{{\rm arc}} \ (p=10)$ &  76.0 $\pm$ 2.39		& 86.0 $\pm$ 4.98 (PWGK) \\
			& $w_{{\rm pers}}$ 		&  49.3 $\pm$ 2.72  		& 52.3 $\pm$ 6.60  \\
			& $w_{{\rm one}}$ 		&  53.8 $\pm$ 4.76 		& 55.1 $\pm$ 8.42 \\ \hline
			& $w_{{\rm arc}} \ (p=5)$ 	&  49.3 $\pm$ 6.92		& 51.8 $\pm$ 3.52 \\
$k_{{\rm L}}$	& $w_{{\rm pers}}$ 		&  51.0 $\pm$ 6.84		& 55.7 $\pm$ 8.68 \\
			& $w_{{\rm one}}$ 		&  50.5 $\pm$ 6.90		& 53.0 $\pm$ 4.89 \\ \hline
\multicolumn{2}{c|}{\textbf{PWGK with Persistence image}} &  	&    \\  
$M=20$		& $w_{{\rm arc}} \ (p=5)$	&  48.8 $\pm$ 3.75 ($\triangle_{20}$)		& 52.0 $\pm$ 5.65 ($\square_{20}$) \\
$M=50$		& $w_{{\rm arc}} \ (p=5)$ 	&  49.2 $\pm$ 5.77 ($\triangle_{50}$)		& 51.8 $\pm$ 7.23 ($\square_{50}$) \\
$M=100$		& $w_{{\rm arc}} \ (p=5)$ 	&  75.0 $\pm$ 2.20 ($\triangle_{100}$)		& 85.8 $\pm$ 4.15 ($\square_{100}$) \\  \hline
\multicolumn{2}{c|}{\textbf{PSSK} }				&  50.5 $\pm$ 5.60 ($K_{{\rm PSS}}$) 	& 53.6 $\pm$ 6.69 \\ \hline
\multicolumn{2}{c|}{\textbf{Persistence landscape}}   &   50.6 $\pm$ 5.92 ($K_{{\rm PL}}$) 	& 48.8 $\pm$ 4.25   \\  \hline
\multicolumn{2}{c|}{\textbf{Persistence image}} &  	&    \\  
$M=20$		& $w_{{\rm pers}}$	 	&  51.1 $\pm$ 4.38  $(K_{{\rm PI}}$)	& 51.7 $\pm$ 6.86  \\
$M=50$		& $w_{{\rm pers}}$	 	&  49.0 $\pm$ 6.14  $(K_{{\rm PI}}$)	& 52.3 $\pm$ 7.21  \\
$M=100$		& $w_{{\rm pers}}$	 	&  54.5 $\pm$ 8.76  $(K_{{\rm PI}}$)	& 52.1 $\pm$ 6.70  \\ \hline
\end{tabular}
\end{table}
In Table \ref{table:Synth_results}, we can see that the PWGK $\triangle$ and the Gaussian kernel on the persistence image with $w_{{\rm arc}}$ and large mesh size $\square_{100}$ show higher classification rates ($85\%$ accuracy) than the other methods ($K_{{\rm PSS}}: 50\%$, $K_{{\rm PL}}: 50\%$, and $K_{{\rm PI}}: 55\%$).
Although the $(k_{{\rm G}},w_{{\rm pers}})$-Gaussian kernel and the persistence image with the original weight $w_{{\rm pers}}$ discount noisy generators, the classification rates are close the chance level.
These unfavorable results must be caused by the difficulty in handling the local and global locations of generators simultaneously.
While the result of the persistence image with a large mesh size is similar to that of the PWGK (e.g., $\square$ and $\square_{100}$), a small mesh size gives bad approximation results (e.g., $\square$ and $\square_{50}$).
The reason is because a small mesh size makes rough pixels, and $S_{2}$ itself and noisy generators are treated in some rough pixel.
On the other hand, we remark that a large mesh size $M$ needs much computational time since the computational complexity of the persistence image depends on $O(M^{2})$.  

We observe that the classification accuracies are not sensitive to $p$. 
Thus, in the rest of this paper, we set $p=5$ because the assumption $p>d+1$ in Theorem \ref{thm:kernel_stability} ensures the continuity in the kernel embedding of persistence diagrams and all data points are obtained from $\lR^{3}$.

\subsection{Analysis of granular system}
\label{subsec:granular} 

We apply the PWGK, the PSSK, the persistence landscape, and the persistence image to persistence diagrams obtained by experimental data in a granular packing system \cite{FSCS11}.
In this example, a partially crystallized packing with $150,000$ monosized beads (diameter $=1$mm, polydispersity $=0.025$mm) in a container is obtained by experiments, where the configuration of the beads is imaged by means of X-ray Computed Tomography.
One of the fundamental interests in the study of granular packings is to understand the transition from random packings to crystallized packings.
In particular, the maximum packing density $\phi_*$ that random packings can attain is still a controversial issue (e.g., see \cite{TTD00}).
Here, we apply the change point analysis to detect $\phi_*$.  

In oder to observe configurations of various densities, we divide the original full system into $35$ cubical subsets containing approximately $4000$ beads.
The data are provided by the authors of the paper \cite{FSCS11}.
The packing densities of the subsets range from $\phi=0.590$ to $\phi=0.730$.
\cite{STRFH17} computed a persistence diagram for each subset by taking the beads configuration as a finite subset in $\mathbb{R}^3$, and found that the persistence diagrams characterize different configurations in random packings (small $\phi$) and crystallized packings (large $\phi$).
Hence, it is expected that the change point analysis applied to these persistence diagrams can detect the maximum packing density $\phi_*$ as a transition from the random to crystallized packings. 

Our strategy is to regard the maximum packing density as the change point and detect it from a collection $\cD=\{D_{\ell} \mid \ell=1,\ldots,n\} \ (n=35)$ of persistence diagrams made by beads configurations of granular systems, where $\ell$ is the index of the packing densities listed in the increasing order.
As a statistical quantity for the change point detection, we use the kernel Fisher discriminant ratio \cite{HMB09} defined by
\begin{equation}
\label{eq:kfdr}
\KFDR_{n,\ell,\gamma}(\cD)=\frac{\ell(n- \ell )}{n} \norm{ \pare{\frac{ \ell }{n} \hat{\Sigma}_{1: \ell }+\frac{n- \ell }{n} \hat{\Sigma}_{ \ell+1:n} +\gamma I}^{-\frac{1}{2}} \pare{\hat{\mu}_{ \ell+1:n}-\hat{\mu}_{1: \ell }} }_{\cH_{K}},
\end{equation}
where the empirical mean element $\hat{\mu}_{i:j}$ and the empirical covariance operator $\hat{\Sigma}_{i:j}$ with data $D_{i}$ through $D_{j} \ (i<j)$ are given by
\begin{align*}
&\hat{\mu}_{i:j}=\frac{1}{j-i+1} \sum^{j}_{\ell=i} K(\cdot,D_{\ell}), \\
&\hat{\Sigma}_{i:j}=\frac{1}{j-i+1}\sum^{j}_{\ell=i} \pare{K(\cdot,D_{\ell})-\hat{\mu}_{i:j} } \otimes  \pare{K(\cdot,D_{\ell})-\hat{\mu}_{i:j} }
\end{align*}
respectively, and $\gamma$ is a regularization parameter (in this paper we set $\gamma=10^{-3}$).
The index $\ell$ achieving the maximum of $\KFDR_{n,\ell,\gamma}(\cD)$ corresponds to the estimated change point. 
In Figure \ref{fig:packing_KFDR}, all the four methods detect $\ell=23$ as the sharp maximizer of the KFDR. This result indicates that the maximum packing density $\phi_*$ exists in the interval $[0.604,0.653]$ and supports the traditional observation $\phi_* \approx 0.636$ \cite{An72}.

\begin{figure}[htbp]
\begin{center}
\includegraphics[width=0.9\hsize]{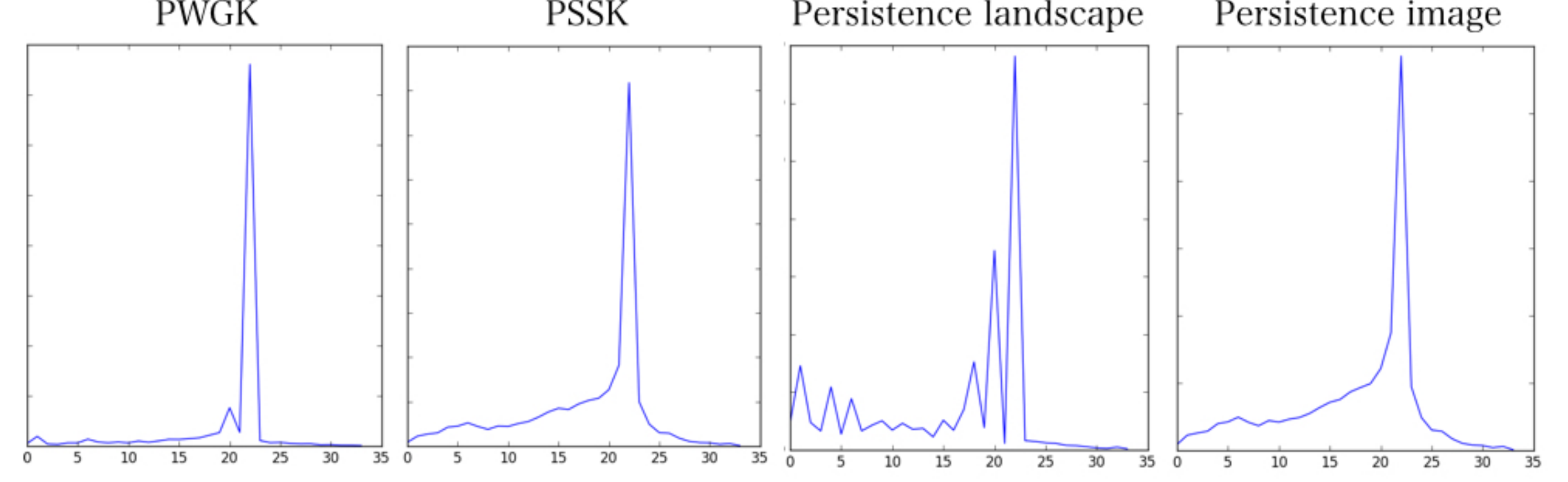}
\end{center}
\vspace{-3mm}
\caption{The $\KFDR$ graphs of the PWGK, the PSSK, the persistence landscape, and the persistence image.}
\label{fig:packing_KFDR}
\end{figure}

We also apply kernel principal component analysis (KPCA) to the same collection of the 35 persistence diagrams. 
Figure \ref{fig:packing_kpca} shows the $2$-dimensional KPCA plots where each green triangle (resp. red circle) indicates the persistence diagram of random packing (resp. crystallized packing).
We can see clear two-cluster structure corresponding to two physical states.

\begin{figure}[htbp]
\begin{center}
\includegraphics[width=0.9\hsize]{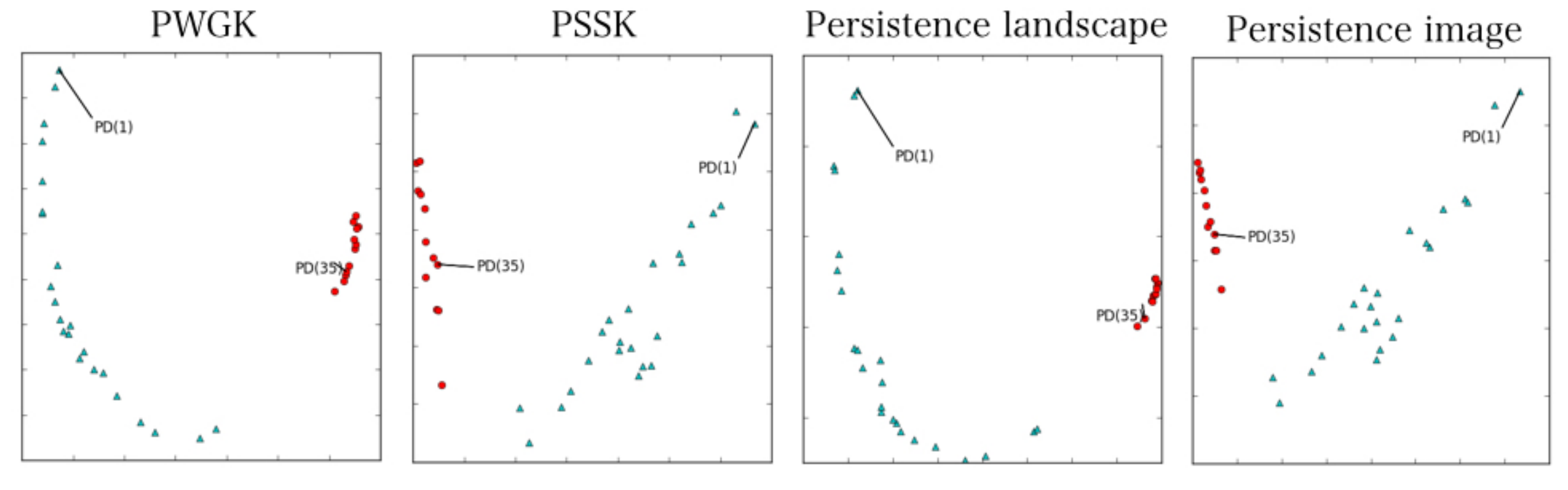}
\end{center}
\vspace{-3mm}
\caption{The KPCA plots of the PWGK (contribution rate: 92.9\%), the PSSK (99.7\%), the persistence landscape (83.8\%), and the persistence image (98.7\%).}
\label{fig:packing_kpca}
\end{figure}

\subsection{Analysis of ${\rm {\bf SiO_2}}$}
\label{subsec:glass}

When we rapidly cool down the liquid state of ${\rm SiO}_2$, it avoids the usual crystallization and changes into a glass state. 
Understanding the liquid-glass transition is an important issue for the current physics and industrial applications \cite{GS07}.
Glass is an amorphous solid, which does not have a clear structure in the configuration of molecules, but it is also known that the medium distance structure such as rings have important influence on the physical properties of the material.
It is thus promising to apply the persistent homology to express the topological and geometrical structure of the glass configuration.
For estimating the glass transition temperature by simulations, a traditional physical method is to prepare atomic configurations of ${\rm SiO_{2}}$ for a certain range of temperatures by molecular dynamics simulations, and then draw the temperature-enthalpy graph. 
The graph consists of two lines in high and low temperatures with slightly different slopes which correspond to the liquid and the glass states, respectively, and the glass transition temperature is conventionally estimated as an interval of the transient region combining these two lines
 (e.g., see \cite{El90}).
However, since the slopes of two lines are close to each other, determining the interval is a subtle problem.
Usually only the rough estimate of the interval is available. Hence, we apply our framework of topological data analysis with kernels to detect the glass transition temperature. 

Let $\{D_\ell\mid \ell=1,\dots,80\}$ be a collection of the persistence diagrams made by atomic configurations of ${\rm SiO}_2$ and sorted by the decreasing order of the temperature. The same data was used in the previous works by \cite{HNHEMN16,NHHEN15}. The interval of the glass transition temperature $T$ estimated by the conventional method explained above is $2000K\leq T\leq 3500K$, which corresponds to $35\leq \ell \leq 50$. 

\begin{figure}[htbp]
\begin{center}
\includegraphics[width=0.9\hsize]{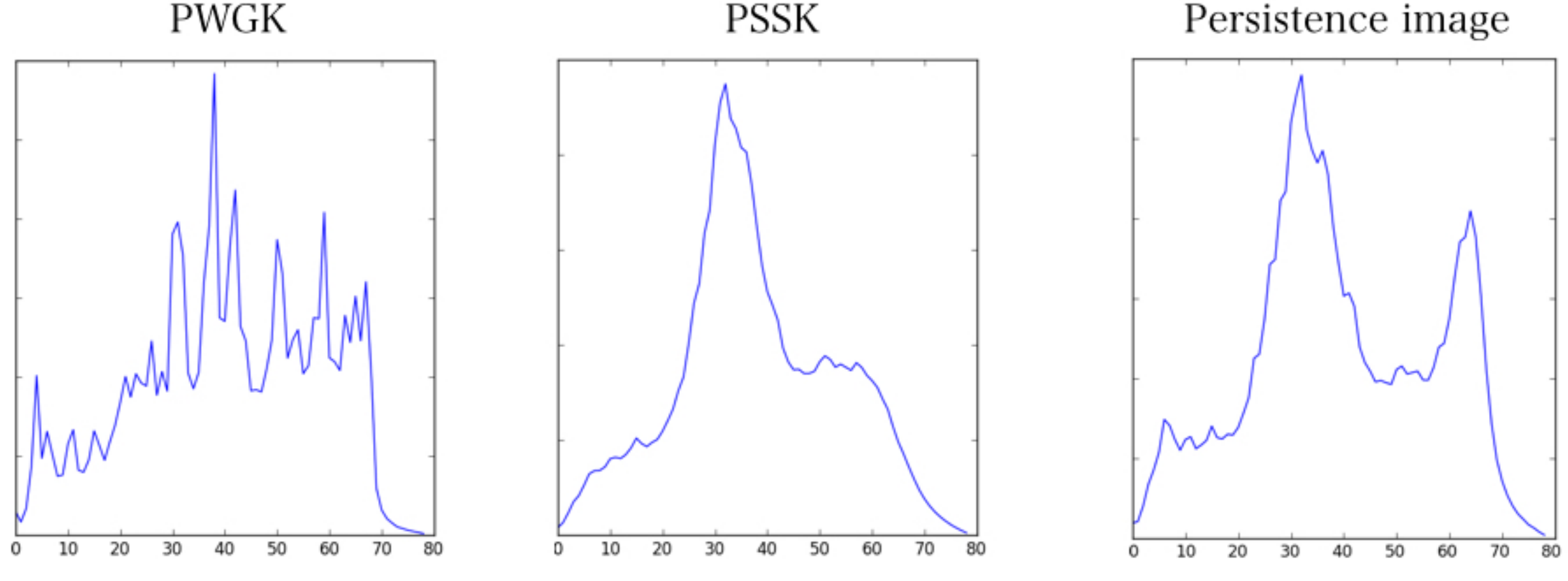}
\end{center}
\vspace{-3mm}
\caption{The $\KFDR$ graphs of the PWGK (left), the PSSK (center) and the persistence image (right).}
\label{fig:glass_KFDR}
\end{figure}
In Figure \ref{fig:glass_KFDR}, the KFDR plots show that the change point is estimated as $\ell=39$ by the PWGK, $\ell=33$ by the PSSK, and $\ell=33$   by the persistence image. For the persistence landscape, we cannot obtain the KFDR or the KPCA results with reasonable computational time.

\begin{figure}[htbp]
\begin{center}
\includegraphics[width=0.9\hsize]{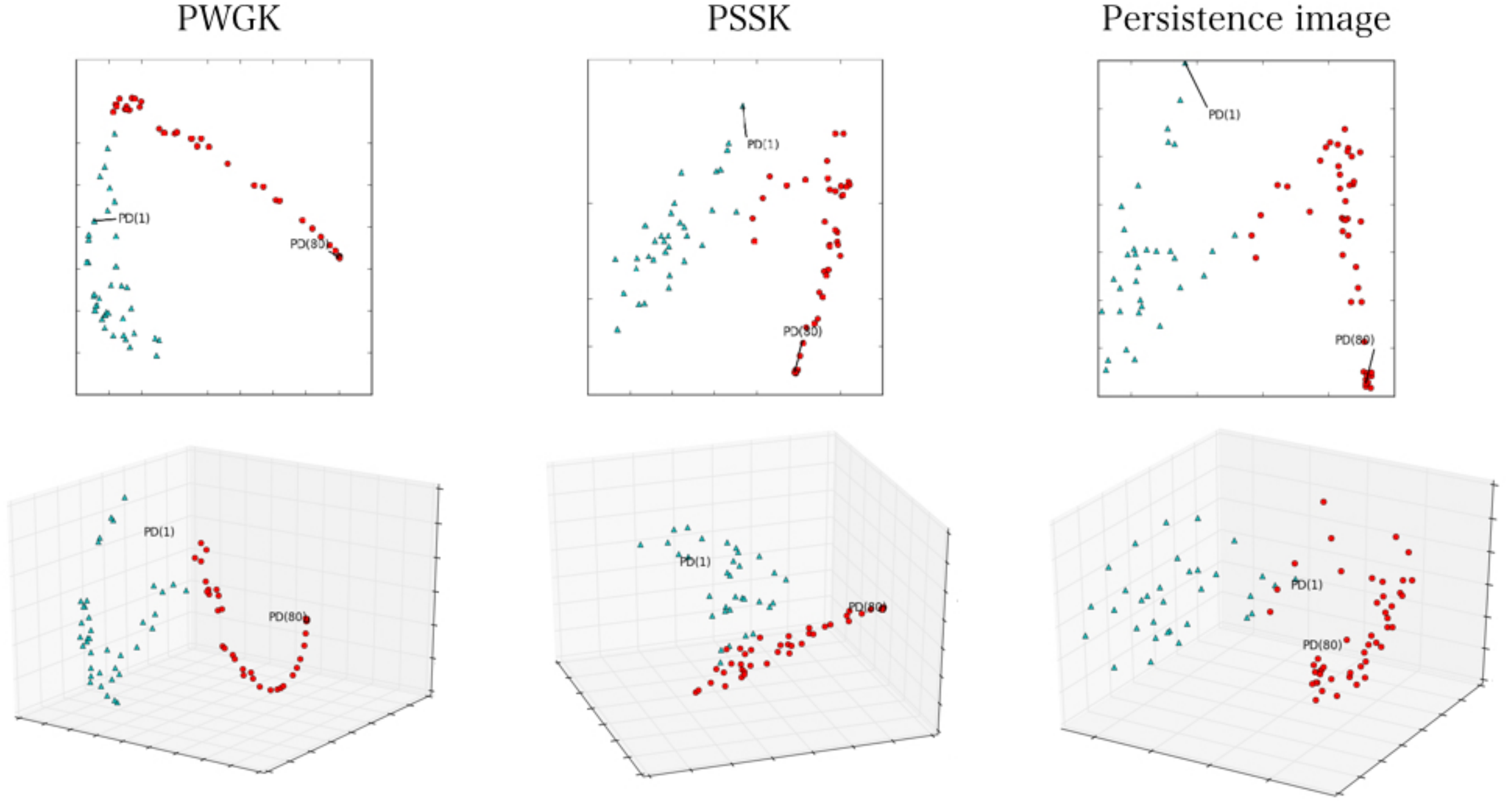}
\end{center}
\vspace{-3mm}
\caption{The $2$-dimensional and $3$-dimensional KPCA plots of the PWGK (contribution rates for $2$-dimension: 81.7\%, $3$-dimension: 92.1\%), the PSSK (97.2\%, 99.3\%) and the persistence image (99.9\%, 99.9\%).}
\label{fig:glass_kpca}
\end{figure}
As we see from the $2$-dimensional plots given by KPCA (Figure \ref{fig:glass_kpca}), the PWGK presents the clear phase change between before (green triangle) and after (red circle) the change point determined by the KFDR.
This strongly suggests that the glass transition occurs at the detected change point.
On the other hand, we cannot observe clear two-cluster structure in the KPCA plots of the PSSK and the persistence image.  
We also remark that the detailed cluster structure is observed in the $3$-dimensional KPCA plots of the PWGK.

\subsection{Protein classification}
\label{subsec:Protein}

We apply the PWGK to two classification tasks studied in \cite{CMWOXW15}.
They introduced the molecular topological fingerprint (MTF) as a feature vector constructed from the persistent homology, and used it for the input to the SVM.
The MTF is given by the $13$-dimensional vector whose elements consist of the persistences of some specific generators\footnote{The MTF method is not a general method for persistence diagrams because some elements of the MTF vector are specialized for protein data, e.g., the ninth element of the MTF vector is defined by the number of Betti $1$ bars that locate at $[4.5, 5.5]$\AA, divided by the number of atoms. For the details, see \cite{CMWOXW15}.} in persistence diagrams.
We compare the performance between the PWGK and the MTF method under the same setting of the SVM reported in \cite{CMWOXW15}.

The first task is a protein-drug binding problem, where the binding and non-binding of drug to the M2 channel protein of the influenza A virus is to be classified.
For each of the two forms, 15 data were obtained by NMR experiments, and 10 data are used for training and the remaining for testing.
We randomly generate 100 ways of partitions and calculate the average classification rates.

In the second problem, the taut and relaxed forms of hemoglobin are to be classified.
For each form, 9 data were collected by the X-ray crystallography.
We select one data from each class for testing and use the remaining  for training.
All the 81 combinations are performed to calculate the CV classification rates.

The results of the two problems are shown in Table \ref{table:Protein_results}.
We can see that the PWGK achieves better performance than the MTF in both problems.
\begin{table}[ttt]
\caption{CV classification rates ($\%$) of SVMs with the PWGK and the MTF (cited from \cite{CMWOXW15}).}
\label{table:Protein_results}
\centering
\begin{tabular}{c|c c}
\hline
& Protein-Drug  &  Hemoglobin \\ \hline
PWGK  &  100  & 88.90  \\
MTF  &  (nbd) 93.91 / (bd) 98.31 & 84.50 \\
\hline
\end{tabular}
\vspace{-3mm}
\end{table}

\section{Conclusion and Discussions}

One of the contributions of this paper is to introduce a kernel framework to topological data analysis with persistence diagrams.
We applied the kernel embedding approach to vectorize the persistence diagrams, which enables us to utilize any standard kernel methods for data analysis.
Another contribution is to propose a kernel specific to persistence diagrams, that is called persistence weighted Gaussian kernel (PWGK).
As a significant advantage, our kernel enables one to control the effect of persistence in data analysis.
We have also proven the stability property with respect to the distance in the Hilbert space.
Furthermore, we have analyzed the synthesized and real data by using the proposed kernel.
The change point detection, the principal component analysis, and the support vector machine derived meaningful results for the tasks. From the viewpoint of computations, our kernel can utilize an efficient approximation to compute the Gram matrix.

One of the main theoretical results of this paper is the stability of the PWGK (Theorem \ref{thm:kernel_stability}).
It is obtained as a corollary of Proposition \ref{prop:general_stability} by restricting the class of persistence diagrams to that obtained from ball model filtrations.
The reason of this restriction is because the total persistence can be bounded from above independent of the persistence diagram.
Thus, one direction to extend this work is to examine the boundedness condition about the total persistence of other persistence diagrams, for example obtained from sub-level sets or Rips complexes.

Another direction to extend this work is to generalize the class of weight functions.
The reason of the choice of $w_{{\rm arc}}$ is mainly for the stability property, but in principle, we can apply any weight function to data analysis.
Then, the question is what types of weight functions have a stability property with respect to the bottleneck or $p$-Wasserstein distance.
Even if we do not concern about stability properties, which weight function is practically good for data analysis?
Suppose generators close to the diagonal are sometimes seen as important features.
Then, our statistical framework can treat such small generators as significant ones by a weight function which has large weight close to the diagonal, while other statistical methods for persistence diagrams always see small generators as noisy ones.

\section*{Acknowledgement}
We thank Ulrich Bauer for giving us useful comments in Section \ref{subsubsec:pssk}, and Mohammad Saadatfar and Takenobu Nakamura for providing experimental and simulation data used in Section \ref{subsec:granular} and \ref{subsec:glass}.
This work is partially supported by JST CREST Mathematics (15656429), JSPS KAKENHI Grant Number 26540016, Structural Materials for Innovation Strategic Innovation Promotion Program D72, Materials research by Information Integration” Initiative (MI$^{2}$I) project of the Support Program for Starting Up, Innovation Hub from JST, and JSPS Research Fellow (17J02401).

\newpage

\appendix

\section{Topological tools}
\label{sec:topology}
This section summarizes some topological tools used in the paper.
To study topological properties algebraically, simplicial complexes are often considered as basic objects.
We start with a brief explanation of simplicial complexes, and gradually increase the generality from simplicial homology to singular and persistent homology.
For more details, see \cite{Ha02}.

\subsection{Simplicial complex}\label{sec:sc}
We first introduce a combinatorial geometric model called simplicial complex to define homology.
Let $P=\{1,\dots,n\}$ be a finite set (not necessarily points in a metric space).
A {\em simplicial complex} with the vertex set $P$ is defined by a collection $S$ of subsets in $P$ satisfying the following properties:
\begin{enumerate}
\item $\{i\}\in S$ for $i=1,\dots,n$, and
\item if $\sigma\in S$ and $\tau\subset \sigma$, then $\tau\in S$.
\end{enumerate}

Each subset $\sigma$ with $q+1$ vertices is called a $q$-simplex.
We denote the set of $q$-simplices by $S_q$.
A subcollection $T\subset S$ which also becomes a simplicial complex (with possibly less vertices) is called a subcomplex of $S$.

We can visually deal with a simplicial complex $S$ as a polyhedron by pasting simplices in $S$ into a Euclidean space.
The simplicial complex obtained in this way is called a geometric realization, and its polyhedron is denoted by $|S|$. 
In this context, the simplices with small $q$ correspond to points ($q=0$), edges ($q=1$), triangles ($q=2$), and tetrahedra ($q=3$). 
\begin{exam}\label{exam:sc}
{\rm 
Figure \ref{fig:sc} shows two polyhedra of simplicial complexes
\begin{align*}
&S=\{
\{1\},
\{2\},
\{3\},
\{1,2\},
\{1,3\},
\{2,3\},
\{1,2,3\}\},\\
&T=\{
\{1\},
\{2\},
\{3\},
\{1,2\},
\{1,3\},
\{2,3\}\}.
\end{align*}

\begin{figure}[htbp]
\begin{center}
\includegraphics[width=0.3\hsize]{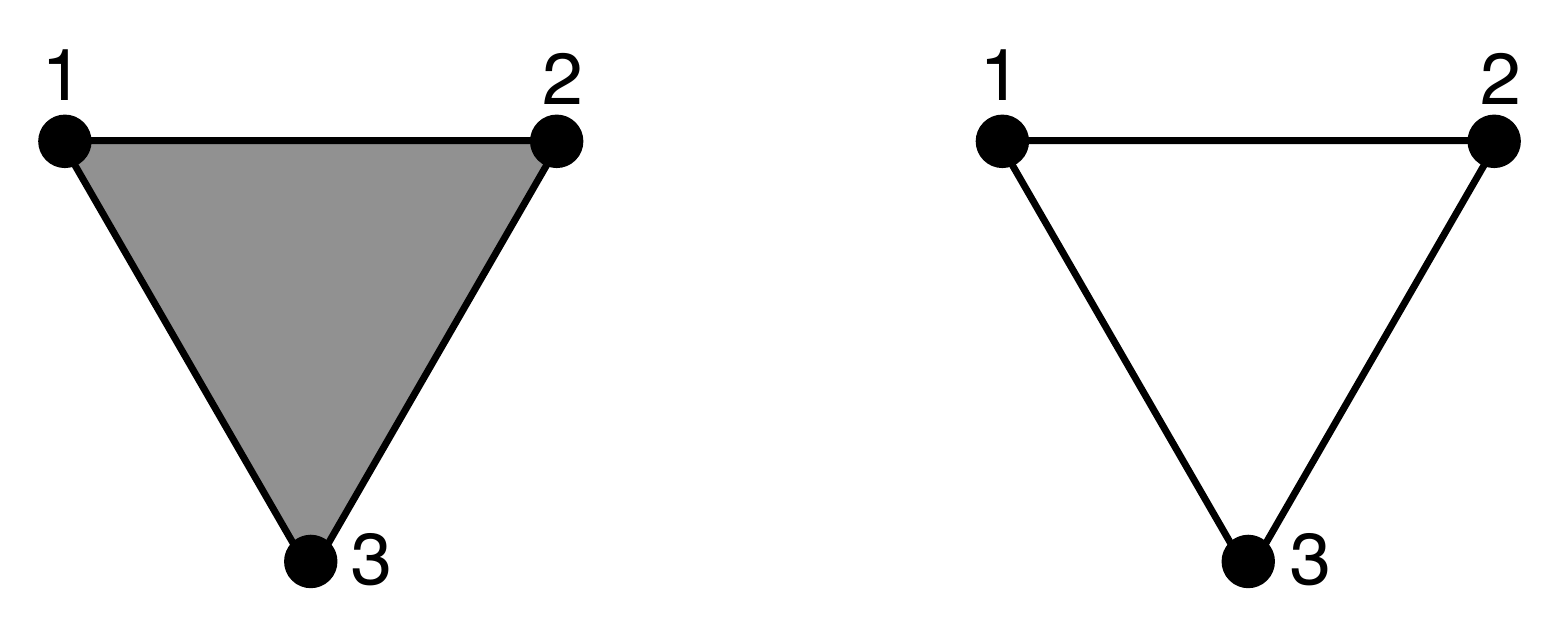}
\end{center}
\caption{The polyhedra of the simplicial complexes $S$ (left) and $T$ (right).}
\label{fig:sc}
\end{figure}
}\end{exam}

\subsection{Homology}\label{sec:homology}
\subsubsection{Simplicial homology}\label{sec:simplicial_homology}
The procedure to define homology is summarized as follows:
\begin{enumerate}
\item Given a simplicial complex $S$, build a chain complex  $C_*(S)$. This is an algebraization of $S$ characterizing the boundary. 
\item Define homology by quotienting out  certain subspaces in $C_*(S)$ characterized by the boundary. 
\end{enumerate}

We begin with the procedure 1 by assigning orderings on simplices. 
When we deal with a $q$-simplex $\sigma=\{i_0,\dots,i_q\}$ as an ordered set, there are $(q+1)!$ orderings on $\sigma$.
For $q>0$, we define an equivalence relation $i_{j_0},\dots,i_{j_q}  \sim  i_{\ell_0},\dots,i_{\ell_q} $ on two orderings of $\sigma$ such that they are mapped to each other by even permutations. 
By definition, two equivalence classes exist, and each of them is called an oriented simplex. 
An oriented simplex is denoted by $\langle i_{j_0},\dots,i_{j_q} \rangle$, and its opposite orientation is expressed by adding the minus $-\langle i_{j_0},\dots,i_{j_q} \rangle$.
We write $\langle\sigma\rangle=\langle i_{j_0},\dots,i_{j_q} \rangle$ for the equivalence class including $i_{j_0}<\dots<i_{j_q}$. For $q=0$, we suppose that we have only one orientation for each vertex. 

Let $K$ be a field. We construct a $K$-vector space $C_q(S)$ as 
\begin{align*}
C_q(S)={\rm Span}_K\{\langle \sigma\rangle\mid \sigma \in S_q\}
\end{align*}
for $S_q\neq\emptyset$ and $C_q(S)=0$ for $S_q=\emptyset$.
Here, ${\rm Span}_K(A)$ for a set $A$ is a vector space over $K$ such that the elements of $A$ formally form a basis of the vector space.
Furthermore, we define a linear map called the {\em boundary map} $\partial_q: C_q(S)\rightarrow C_{q-1}(S)$ by the linear extension of 
\begin{align}\label{eq:boundary}
\partial_q\langle i_0,\dots,i_q\rangle=\sum_{\ell=0}^q(-1)^\ell\langle i_0,\dots,\widehat{i_\ell},\dots,i_q\rangle,
\end{align}
where $\widehat{i_\ell}$ means the removal of the vertex $i_\ell$. We can regard the linear map $\partial_q$ as algebraically capturing the $(q-1)$-dimensional boundary of a $q$-dimensional object. 

For example, the image of the $2$-simplex $\langle\sigma\rangle=\langle 1,2,3\rangle$ is given by $\partial_2\langle\sigma\rangle=\langle2,3\rangle-\langle1,3\rangle+\langle1,2\rangle$, which is the boundary of $\sigma$ (see Figure \ref{fig:sc}). 

In practice, by arranging some orderings of the oriented $q$- and $(q-1)$- simplices, we can represent the boundary map as a matrix $M_q=(M_{\sigma,\tau})_{\sigma\in S_{q-1},\tau\in S_q}$ with the entry $M_{\sigma,\tau}=0,\pm 1$ given by the coefficient in \eqref{eq:boundary}.
For the simplicial complex $S$ in Example \ref{exam:sc}, the matrix representations $M_1$ and $M_2$ of the boundary maps are  given by 
\begin{align}\label{eq:matrix}
M_2=\left[
\begin{array}{r}
1\\
1\\
-1
\end{array}\right],\quad
M_1=\left[
\begin{array}{rrr}
-1& 0&-1\\
 1&-1&0\\
 0& 1&1
\end{array}
\right]
\end{align}
Here, the $1$-simplices (resp. $0$-simplices) are ordered by $\langle 1,2\rangle, \langle2,3\rangle, \langle 1,3\rangle$ (resp. $\langle1\rangle$, $\langle2\rangle$, $\langle3\rangle$).

We call a sequence of the vector spaces and linear maps
\begin{align*}
\xymatrix{
\cdots\ar[r] & C_{q+1}(S)\ar[r]^{\partial_{q+1}} & C_q(S)\ar[r]^{\partial_q} & C_{q-1}(S)\ar[r] & \cdots
}
\end{align*}
the {\em chain complex} of $S$. As an easy exercise, we can show $\partial_{q}\circ \partial_{q+1}=0$.
Hence, the subspaces $Z_q(S)={\rm ker}\partial_q$ and $B_q(S)={\rm im}\partial_{q+1}$ satisfy $B_q(S)\subset Z_q(S)$. Then, the $q$-th (simplicial) {\em homology} is defined by taking the quotient space
\begin{align*}
H_q(S)=Z_q(S)/B_q(S).
\end{align*}
Intuitively, the dimension of $H_q(S)$ counts the number of $q$-dimensional holes in $S$ and each generator of the vector space $H_q(S)$ corresponds to these holes.
We remark that the homology as a vector space is independent of the orientations of simplices. 

For a subcomplex $T$ of $S$, the inclusion map $\rho:T\hookrightarrow S$ naturally induces a linear map in homology $\rho_q: H_q(T)\rightarrow H_q(S)$. Namely, an element $[c]\in H_q(T)$ is mapped to $[c]\in H_q(S)$, where the equivalence class $[c]$ is taken in each vector space. 

For example, the simplicial complex $S$ in Example \ref{exam:sc} has 
\[
Z_1(S)={\rm Span}_K[\begin{array}{ccc}1 & 1 & -1\end{array}]^T=B_1(S)
\]
 from (\ref{eq:matrix}). Hence $H_1(S)=0$, meaning that there are no $1$-dimensional hole (ring) in $S$.
On the other hand, since $Z_1(T)=Z_1(S)$ and $B_1(T)=0$, we have $H_1(T)\cong K$, meaning that $T$ consists of one ring.
Hence, the induced linear map $\rho_1: H_1(T)\rightarrow H_1(S)$ means that the ring in $T$ disappears in $S$ under $T\hookrightarrow S$.

A topological space $X$ is called {\em triangulable} if there exists a geometric realization of a simplicial complex $S$ whose polyhedron is homeomorphic\footnote{A continuous map $f:X \ra Y$ is said to be {\em homeomorphic} if $f:X \ra Y$ is bijective and the inverse $f^{-1}:Y \ra X$ is also continuous.} to $X$.
For such a triangulable topological space, the homology is defined by $H_q(X)=H_q(S)$.
This is well-defined, since a different geometric realization provides an isomorphic homology. 

\subsubsection{Singular homology}\label{sec:singular_homology}
We here extend the homology to general topological spaces.
Let $e_0,\dots,e_q$ be the standard basis of ${\mathbb R}^{q+1}$ (i.e., $e_i=(0,\dots,0,1,0,\dots,0)$, 1 at  $(i+1)$-th position, and 0 otherwise), and set
\begin{align*}
&\Delta_q=\rl{
\sum_{i=0}^q\lambda_ie_i \lmid \sum_{i=0}^q\lambda_i=1, \lambda_i\geq 0
},\\
&\Delta^\ell_q=\rl{
\sum_{i=0}^q\lambda_ie_i \lmid \sum_{i=0}^q\lambda_i=1, \lambda_i\geq 0, \lambda_\ell=0
}.
\end{align*}
We also denote the inclusion by $\iota^\ell_q: \Delta^\ell_q\hookrightarrow \Delta_q$. 

For a topological space $X$, a continuous map $\sigma: \Delta_q\rightarrow X$ is called a singular $q$-simplex, and let $X_q$ be the set of $q$-simplices. 
We construct a $K$-vector space $C_q(X)$ as 
\begin{align*}
C_q(X)={\rm Span}_K\{\sigma\mid\sigma\in X_q\}.
\end{align*}
The boundary map $\partial_q:C_q(X)\rightarrow C_{q-1}(X)$ is defined by the linear extension of
\begin{align*}
\partial_q \sigma=\sum_{\ell=0}^q(-1)^\ell\sigma\circ \iota^\ell_q.
\end{align*}

Even in this setting, we can show that $\partial_{q}\circ\partial_{q+1}=0$, and hence the subspaces $Z_q(X)={\rm ker}\partial_q$ and $B_q(X)={\rm im}\partial_{q+1}$ satisfy $B_q(X)\subset Z_q(X)$. Then, the $q$-th (singular) {\em homology} is similarly defined by 
\begin{align*}
H_q(X)=Z_q(X)/B_q(X).
\end{align*}
It is known that, for a triangulable topological space, the homology of this definition is isomorphic to that defined in 
\ref{sec:simplicial_homology}. From this reason, we hereafter identify simplicial and singular homology. 

The induced linear map in homology for an inclusion pair of topological space $Y\subset X$ is similarly defined as in \ref{sec:simplicial_homology}.

\section{Total persistence}
\label{sec:total}

Let $(M,d_{M})$ be a triangulable compact metric space.
For a Lipschitz function $f:M \ra \lR$, we define the degree-$p$ total persistence over $t$ by 
\begin{equation*}
\Pers_{p}(D_{q}(\Sub(f)),t)=\sum_{\substack{x \in D_{q}(\Sub(f)) \\ \pers(x)>t}} \pers(x)^{p} 
\end{equation*}
for $0 \leq t \leq \Amp(f)$, where $\Amp(f) := \max_{\bm{x} \in M}f(\bm{x})-\min_{\bm{x} \in M} f(\bm{x})$ is the amplitude of $f$.
Let $S$ be a triangulated simplicial complex of $M$ by a homeomorphism $\vartheta: \abs{S} \ra M$.
The diameter of a simplex $\sigma \in S$ and the mesh of the triangulation $S$ are defined by $\diam(\sigma) = \max_{\bm{x},\bm{y} \in \sigma} d_{M}(\vartheta(\bm{x}),\vartheta(\bm{y}))$ and $\mesh(S) = \max_{\sigma \in S} \diam(\sigma )$, respectively.
Furthermore, let us set $N(r)= \min_{\mesh(S ) \leq r} \card{S}$.
Then, the degree-$p$ total persistence over $t$ is bounded from above as follows:
\begin{lem}[\cite{CEHM10}]
Let $M$ be a triangulable compact metric space and $f:M \ra \lR$ be a tame Lipschitz function.
Then, $\Pers_{p}(D_{q}(\Sub(f)),t)$ is bounded from above by
\[
t^{p}N \pare{\frac{t}{\Lip(f)}} + p \int^{\Amp(f)}_{\ee=t} N \pare{\frac{\ee}{\Lip(f)}} \ee^{p-1}d\ee,
\]
where $\Lip(f)$ is the Lipschitz constant of $f$.
\end{lem}

For a compact triangulable subspace $M$ in $\lR^{d}$, the number of $d$-cubes with length $r>0$ covering $M$ is bounded from above by $O(\frac{1}{r^{d}})$, and hence there exists some constant $C_{M}$ depending only on $M$ such that $N(r) \leq \frac{C_{M}}{r^{d}}$.

For $p>d$, we can find the upper bounds for the both terms as follows:
\begin{align*}
t^{p}N\pare{\frac{t}{\Lip(f)}} \leq t^{p}C_{M} \frac{ \Lip(f)^{d}} {t^{d}} \ \ra \ 0 ~~ (t \ra \infty)
\end{align*}
and
\begin{align*}
p \int^{\Amp(f)}_{\ee=t} N\pare{\frac{\ee}{\Lip(f)}} \ee^{p-1}d\ee \leq \frac{p}{p-d} C_{M} \Lip(f)^{d} \Amp(f)^{p-d} .
\end{align*}
Then, the upper bound of the total persistence $\Pers_{p}(D_{q}(\Sub(f)))=\Pers_{p}(D_{q}(\Sub(f)),0)$ is given as follows:
\begin{lem}
\label{lem:total}
Let $M$ be a triangulable compact subspace in $\lR^{d}$ and $p>d$.
For any Lipschitz function $f:M \ra \lR$, 
\[
\Pers_{p}(D_{q}(\Sub(f))) \leq \frac{p}{p-d}C_{M} \Lip(f)^{d} \Amp(f)^{p-d},
\]
where $C_{M}$ is a constant depending only on $M$.
\end{lem}

In the case of a finite subset $X \subset \lR^{d}$, there always exists an $R$-ball $M$ containing $X$ for some $R>0$, which is a triangulable compact subspace in $\lR^{d}$.
Moreover, by estimating $\Lip(\dist_{X})^{d} \Amp(\dist_{X})^{p-d}$, we show Lemma \ref{lem:point_total} as a corollary of Lemma \ref{lem:total}:

\begin{proof}[Lemma \ref{lem:point_total}]
The Lipschitz constant of $\dist_{X}$ is $1$, because, for any $\bm{x},\bm{y} \in M$,
\begin{align*}
\dist_{X}(\bm{x})-\dist_{X}(\bm{y}) 
&=\min_{\bm{x}_{i} \in X} d_{M}(\bm{x},\bm{x}_{i})-\min_{\bm{x}_{i} \in X} d_{M}(\bm{y},\bm{x}_{i}) \\
&\leq \min_{\bm{x}_{i} \in X} (d_{M}(\bm{x},\bm{y})+d_{M}(\bm{y},\bm{x}_{i})) - \min_{\bm{x}_{i} \in X} d_{M}(\bm{y},\bm{x}_{i}) \\
&=d_{M}(\bm{x},\bm{y}).
\end{align*}
Moreover, 
\[
\Amp(\dist_{X}) \leq \diam(M):=\max_{\bm{x}_{i},\bm{x}_{i} \in M} d_{M}(\bm{x}_{i},\bm{x}_{i}),
\]
because $\min_{\bm{x} \in M} \dist_{X}(\bm{x})=0$ and $\max_{\bm{x} \in M} \dist_{X}(\bm{x}) \leq \diam(M)$.
Thus, for some constant $C_{M}$ depending only on $M$,  we have
\begin{align*}
\Pers_{p}(D_{q}(X))
&= \Pers_{p}(D_{q}(\Sub(\dist_{X})))\\
&\leq \frac{p}{p-d}C_{M}\Lip(\dist_{X})^{d} \Amp(\dist_{X})^{p-d} \\
&\leq \frac{p}{p-d}C_{M}\diam(M)^{p-d}.
\end{align*}
\end{proof}

For a persistence diagram $D=\{x_{1},\ldots,x_{n} \}$, we construct a $n$-dimensional vector
\[
v(D):=\pare{ \pers(x_{1}),\ldots,\pers(x_{n}) }.
\]
Then, the degree-$p$ total persistence is represented as 
\[
\Pers_{p}(D)=\norm{v(D)}^{p}_{p},
\]
where $\norm{\cdot}_{p}$ denotes the $\ell^{p}$-norm of $\lR^{n}$.
Since $\norm{v}_{q} \leq \norm{v}_{p} \ (v \in \lR^{n}, \ 1 \leq p \leq q < \infty)$, we have
\[
\Pers_{q}(D)^{\frac{1}{q}} =\norm{v(D)}_{q} \leq \norm{v(D)}_{p} = \Pers_{p}(D)^{\frac{1}{p}}.
\]

\begin{prop}
\label{prop:persistence_inequality}
If $1 \leq p \leq q<\infty$ and $\Pers_{p}(D)$ is bounded from above, $\Pers_{q}(D)$ is also bounded from above.
\end{prop}

\section{Lemmata for Proposition \ref{prop:general_stability}}
\label{sec:stability}

\begin{lem}
\label{lemm:Lip_k}
For any $x,y \in \lR^{2}$, $\norm{k_{{\rm G}}(\cdot,x)-k_{{\rm G}}(\cdot,y)}_{\cH_{k_{{\rm G}}}} \leq \frac{\sqrt{2}}{\sigma} \norm{x-y}_{\infty}$.
\end{lem}

\begin{proof}
\begin{align}
\norm{k_{{\rm G}}(\cdot,x)-k_{{\rm G}}(\cdot,y)}^{2}_{\cH_{k_{{\rm G}}}} \nonumber 
&=k_{{\rm G}}(x,x)+k_{{\rm G}}(y,y)-2k_{{\rm G}}(x,y) \nonumber \\
&=1+1-2 e^{-\frac{\norm{x-y}^{2}}{2 \sigma^{2}}} \nonumber \\
&= 2 \pare{ 1-e^{-\frac{\norm{x-y}^{2}}{2 \sigma^{2}}} } \nonumber  \\
& \leq  \frac{1}{\sigma^{2}}\norm{x-y}^{2} \label{eq:eq1} \\
& \leq \frac{2}{\sigma^{2}}\norm{x-y}^{2}_{\infty}. \label{eq:eq2} 
\end{align}
We have used the fact $1-e^{-t} \leq t \ ( t \in \lR)$ in \eqref{eq:eq1} and $\norm{x}^{2} \leq 2 \norm{x}^{2}_{\infty} \ (x \in \lR^{2})$ in \eqref{eq:eq2}.
\end{proof}

\begin{lem}
\label{lemm:persistence}
For any $x,y \in \lR^{2}$, the difference of persistences $\abs{\pers(x)-\pers(y)}$ is less than or equal to $2 \norm{x-y}_{\infty}$.
\end{lem}

\begin{proof}
For $x=(x_{1},x_{2}),y=(y_{1},y_{2})$, we have
\begin{align}
| \pers(x)-\pers(y) |
&= |(x_{2}-x_{1}) - (y_{2}-y_{1})| \nonumber \\
&\leq |x_{2}-y_{2}|+|x_{1}-y_{1}| \nonumber \\
&\leq 2 \norm{x-y}_{\infty}. \nonumber
\end{align}
\end{proof}

\begin{lem}
\label{lemm:w_continuous}
For any $x,y \in \lR^{2}$, we have 
\begin{align*}
\abs{w_{{\rm arc}}(x)-w_{{\rm arc}}(y)}  \leq 2pC \max \{\pers(x)^{p-1}, \pers(y)^{p-1}\} \norm{x-y}_{\infty}.
\end{align*}
\end{lem}

\begin{proof}
\begin{align}
&\abs{w_{{\rm arc}}(x)-w_{{\rm arc}}(y) } \nonumber \\
&=\abs{ \arctan(C \pers(x)^{p})-\arctan(C \pers(y)^{p}) } \label{eq:arctan} \\
& \leq C \abs{ \pers(x)^{p}-\pers(y)^{p}} \nonumber \\
& \leq C \abs{\pers(x)-\pers(y)} p \max \{\pers(x)^{p-1}, \pers(y)^{p-1}\} \label{eq:p}\\
& \leq 2 pC \max \{\pers(x)^{p-1}, \pers(y)^{p-1}\} \norm{x-y}_{\infty} \label{eq:infty_norm}. 
\end{align}
We have used the fact that the Lipschitz constant of $\arctan$ is $1$ in \eqref{eq:arctan},
\begin{align*}
s^{p}-t^{p} 
&=(s-t)(s^{p-1}+s^{p-2}t+\cdots+t^{p-1}) \\
& \leq (s-t) p \max \{ s^{p-1},t^{p-1}\}
\end{align*}
for any $s,t >0$ in \eqref{eq:p}, and Lemma \ref{lemm:persistence} in \eqref{eq:infty_norm}.
\end{proof}

\section{Remark on the bottleneck stability of the PWGK}
\label{sec:additive}
Let $K$ be a positive definite kernel on persistence diagrams.
Then, 
\[
d_{K}(D,E) = \sqrt{K(D,D) + K(E,E) -2K(D,E)}
\]
defines a semi-metric on persistence diagrams.
A positive definite kernel $K$ is said to be {\em additive} if $K(D \cup D', E) = K(D,E) + K(D',E)$ and {\em trivial} if $K(D,E) = 0$ for any persistence diagrams $D,D',E$. 
It is shown that a non-trivial additive kernel does not satisfy the $d_{ {\rm W}_{p}}$ stability for $p>1$ by giving a counterexample.

\begin{prop}[\cite{RHBK15}]
\label{prop:additive}
Let $K$ be a non-trivial additive positive definite kernel $K$ on persistence diagrams such that $K(\cdot, \DD) =0$ for the diagonal set $\DD$.
Then, for any $1 < p \leq \infty$, there exists no constant $L>0$ such that
\[
d_{K}(D,E) \leq L d_{ {\rm W}_{p}}(D,E).
\]
\end{prop}

\begin{proof}
Since $K$ is non-trivial, there exists a persistence diagram $D$ such that $K(D,D) > 0$.
Then, for any $n>0$, we compute both distance between $\cup^{n}_{i =1}D$ and the diagonal set $\DD$:
\begin{align*}
d_{K}(\cup^{n}_{i =1}D, \DD)&=n \sqrt{K(D,D)}, \\
d_{ {\rm W}_{p}}(\cup^{n}_{i =1}D, \DD)&=d_{ {\rm W}_{p}}(D,\DD)
\begin{cases}
\sqrt[n]{n}, & 1 < p < \infty \\
1, & p = \infty
\end{cases},
\end{align*}
Hence, $d_{K}$ cannot be bounded by $L d_{ {\rm W}_{p}}$ with a constant $L>0$.
\end{proof}

Actually, since the kernel $K_{{\rm PWG}}$ defined by $K_{{\rm PWG}}(D, E)=\inn{E_{k}(\mu^{w}_{D})}{E_{k}(\mu^{w}_{E})}_{\cH_{k}}$ is non-trivial, additive, and $E_{k}(\mu^{w}_{\DD})=0$, it seems that the PWGK would not satisfy the bottleneck stability and it contradicts Theorem \ref{thm:kernel_stability}.
However, when this counterexample is applied to Proposition \ref{prop:general_stability}, 
because $\Pers_{p}(\cup^{n}_{i =1}D) = n \Pers_{p}(D)$ and $\Pers_{p}(\DD) = 0$, we obtain
\begin{align*}
\dk{k_{{\rm G}}}{w}{\cup^{n}_{i =1}D}{\DD} 
&= n \sqrt{K_{{\rm PWG}}(D,D)}, \\
L(\cup^{n}_{i =1}D,\DD;C,p,\sigma)
&=nL(D,\DD;C,p,\sigma), \\
d_{{\rm B}}(\cup^{n}_{i =1}D, \DD) 
&= d_{{\rm B}}(D, \DD). 
\end{align*}
In other words, Proposition \ref{prop:general_stability} is not affected by $n$ in $\cup^{n}_{i =1}D$ and Theorem \ref{thm:kernel_stability} does not contradict with Proposition \ref{prop:additive}.

\newpage

\bibliographystyle{alpha}
\bibliography{reference}

\newcommand{\etalchar}[1]{$^{#1}$}
\begin{thebibliography}{CSEHM10}

\bibitem[AEK{\etalchar{+}}17]{AEKNPSCHMZ17}
Henry Adams, Tegan Emerson, Michael Kirby, Rachel Neville, Chris Peterson,
  Patrick Shipman, Sofya Chepushtanova, Eric Hanson, Francis Motta, and Lori
  Ziegelmeier.
\newblock Persistence images: A stable vector representation of persistent
  homology.
\newblock {\em Journal of Machine Learning Research}, 18(8):1--35, 2017.

\bibitem[Ano72]{An72}
Anonymous.
\newblock What is random packing?
\newblock {\em Nature}, 239:488--489, 1972.

\bibitem[BD17]{BD17}
Peter Bubenik and Pawe{\l} D{\l}otko.
\newblock A persistence landscapes toolbox for topological statistics.
\newblock {\em Journal of Symbolic Computation}, 78:91--114, 2017.

\bibitem[BKRW14]{BKRW14}
Ulrich Bauer, Michael Kerber, Jan Reininghaus, and Hubert Wagner.
\newblock {\em Mathematical Software -- ICMS 2014: 4th International Congress,
  Seoul, South Korea, August 5-9, 2014. Proceedings}, chapter PHAT --
  Persistent Homology Algorithms Toolbox, pages 137--143.
\newblock Springer Berlin Heidelberg, Berlin, Heidelberg, 2014.

\bibitem[Bub15]{Bu15}
Peter Bubenik.
\newblock Statistical topological data analysis using persistence landscapes.
\newblock {\em Journal of Machine Learning Research}, 16(1):77--102, 2015.

\bibitem[Car09]{Ca09}
Gunnar Carlsson.
\newblock Topology and data.
\newblock {\em Bulletin of the American Mathematical Society}, 46(2):255--308,
  2009.

\bibitem[CdSO14]{CdSO14}
Fr{\'e}d{\'e}ric Chazal, Vin de~Silva, and Steve Oudot.
\newblock Persistence stability for geometric complexes.
\newblock {\em Geometriae Dedicata}, 173(1):193--214, 2014.

\bibitem[CGLM15]{CGLM15}
Fr{\'e}d{\'e}ric Chazal, Marc Glisse, Catherine Labru{\`e}re, and Bertrand
  Michel.
\newblock Convergence rates for persistence diagram estimation in topological
  data analysis.
\newblock {\em Journal of Machine Learning Research}, 16:3603--3635, 2015.

\bibitem[CIdSZ08]{CIdSZ08}
Gunnar Carlsson, Tigran Ishkhanov, Vin de~Silva, and Afra Zomorodian.
\newblock On the local behavior of spaces of natural images.
\newblock {\em International journal of computer vision}, 76(1):1--12, 2008.

\bibitem[CMW{\etalchar{+}}15]{CMWOXW15}
Zixuan Cang, Lin Mu, Kedi Wu, Kristopher Opron, Kelin Xia, and Guo-Wei Wei.
\newblock A topological approach for protein classification.
\newblock {\em Molecular Based Mathematical Biology}, 3(1), 2015.

\bibitem[COO15]{COO15}
Mathieu Carri\`ere, Steve Oudot, and Maks Ovsjanikov.
\newblock Local signatures using persistence diagrams.
\newblock preprint, 2015.

\bibitem[CSEH07]{CEH07}
David Cohen-Steiner, Herbert Edelsbrunner, and John Harer.
\newblock Stability of persistence diagrams.
\newblock {\em Discrete \& Computational Geometry}, 37(1):103--120, 2007.

\bibitem[CSEHM10]{CEHM10}
David Cohen-Steiner, Herbert Edelsbrunner, John Harer, and Yuriy Mileyko.
\newblock Lipschitz functions have $l_p$-stable persistence.
\newblock {\em Foundations of computational mathematics}, 10(2):127--139, 2010.

\bibitem[DLY15]{DLY15}
Tran Kai~Frank Da, S{\'e}bastien Loriot, and Mariette Yvinec.
\newblock {3D} alpha shapes.
\newblock In {\em {CGAL} User and Reference Manual}. {CGAL Editorial Board},
  {4.7} edition, 2015.

\bibitem[dSG07]{dSG07}
Vin de~Silva and Robert Ghrist.
\newblock Coverage in sensor networks via persistent homology.
\newblock {\em Algebraic \& Geometric Topology}, 7(1):339--358, 2007.

\bibitem[DUJ77]{DU77}
Joseph Diestel and J~Jerry Uhl~Jr.
\newblock Vector measures. with a foreword by bj pettis. mathematical surveys,
  no. 15.
\newblock {\em American Mathematical Society, Providence, RI}, 56:12216, 1977.

\bibitem[Ell90]{El90}
Stephen~R. Elliott.
\newblock {\em Physics of amorphous materials (2nd)}.
\newblock Longman London; New York, 1990.

\bibitem[ELZ02]{ELZ02}
Herbert Edelsbrunner, David Letscher, and Afra Zomorodian.
\newblock Topological persistence and simplification.
\newblock {\em Discrete and Computational Geometry}, 28(4):511--533, 2002.

\bibitem[FLR{\etalchar{+}}14]{FLRWBS14}
Brittany~Terese Fasy, Fabrizio Lecci, Alessandro Rinaldo, Larry Wasserman,
  Sivaraman Balakrishnan, and Aarti Singh.
\newblock Confidence sets for persistence diagrams.
\newblock {\em The Annals of Statistics}, 42(6):2301--2339, 2014.

\bibitem[FSCS13]{FSCS11}
Nicolas Francois, Mohammad Saadatfar, R~Cruikshank, and A~Sheppard.
\newblock Geometrical frustration in amorphous and partially crystallized
  packings of spheres.
\newblock {\em Physical review letters}, 111(14):148001, 2013.

\bibitem[GFT{\etalchar{+}}07]{GFTSSS07}
Arthur Gretton, Kenji Fukumizu, Choon~H. Teo, Le~Song, Bernhard Sch{\"o}lkopf,
  and Alex~J Smola.
\newblock A kernel statistical test of independence.
\newblock In {\em Advances in Neural Information Processing Systems}, pages
  585--592, 2007.

\bibitem[GHI{\etalchar{+}}15]{GHIKMN13}
Marcio Gameiro, Yasuaki Hiraoka, Shunsuke Izumi, Miroslav Kramar, Konstantin
  Mischaikow, and Vidit Nanda.
\newblock A topological measurement of protein compressibility.
\newblock {\em Japan Journal of Industrial and Applied Mathematics},
  32(1):1--17, 2015.

\bibitem[GS07]{GS07}
Neville~G Greaves and Sabyasachi Sen.
\newblock Inorganic glasses, glass-forming liquids and amorphizing solids.
\newblock {\em Advances in Physics}, 56(1):1--166, 2007.

\bibitem[Hat02]{Ha02}
Allen Hatcher.
\newblock {\em Algebraic topology}.
\newblock Cambridge University Press, 2002.

\bibitem[HMB09]{HMB09}
Zaid Harchaoui, Eric Moulines, and Francis~R. Bach.
\newblock Kernel change-point analysis.
\newblock In {\em Advances in Neural Information Processing Systems}, pages
  609--616, 2009.

\bibitem[HNH{\etalchar{+}}16]{HNHEMN16}
Yasuaki Hiraoka, Takenobu Nakamura, Akihiko Hirata, Emerson~G Escolar, Kaname
  Matsue, and Yasumasa Nishiura.
\newblock Hierarchical structures of amorphous solids characterized by
  persistent homology.
\newblock {\em Proceedings of the National Academy of Sciences},
  113(26):7035--7040, 2016.

\bibitem[KHF16]{KFH16}
Genki Kusano, Yasuaki Hiraoka, and Kenji Fukumizu.
\newblock Persistence weighted gaussian kernel for topological data analysis.
\newblock In {\em International Conference on Machine Learning}, pages
  2004--2013, 2016.

\bibitem[KZP{\etalchar{+}}07]{KZPSGP07}
Peter~M Kasson, Afra Zomorodian, Sanghyun Park, Nina Singhal, Leonidas~J
  Guibas, and Vijay~S Pande.
\newblock Persistent voids: a new structural metric for membrane fusion.
\newblock {\em Bioinformatics}, 23(14):1753--1759, 2007.

\bibitem[LCK{\etalchar{+}}11]{LCKKL11}
Hyekyoung Lee, Moo~K Chung, Hyejin Kang, Bung-Nyun Kim, and Dong~Soo Lee.
\newblock Discriminative persistent homology of brain networks.
\newblock In {\em Biomedical Imaging: From Nano to Macro, 2011 IEEE
  International Symposium on}, pages 841--844. IEEE, 2011.

\bibitem[MFDS12]{MFDS12}
Krikamol Muandet, Kenji Fukumizu, Francesco Dinuzzo, and Bernhard
  Sch{\"o}lkopf.
\newblock Learning from distributions via support measure machines.
\newblock In {\em Advances in neural information processing systems}, pages
  10--18, 2012.

\bibitem[MFSSar]{MFSS17}
Krikamol Muandet, Kenji Fukumizu, Bharath~K. Sriperumbudur, and Bernhard
  Sch{\"o}lkopf.
\newblock Kernel mean embedding of distributions: A review and beyonds.
\newblock {\em Foundations and Trends in Machine Learning}, To appear.

\bibitem[NHH{\etalchar{+}}15]{NHHEN15}
Takenobu Nakamura, Yasuaki Hiraoka, Akihiko Hirata, Emerson~G Escolar, and
  Yasumasa Nishiura.
\newblock Persistent homology and many-body atomic structure for medium-range
  order in the glass.
\newblock {\em Nanotechnology}, 26(304001), 2015.

\bibitem[PET{\etalchar{+}}14]{PETCNHV14}
Giovanni Petri, Paul Expert, Federico Turkheimer, Robin Carhart-Harris, David
  Nutt, Peter~J Hellyer, and Francesco Vaccarino.
\newblock Homological scaffolds of brain functional networks.
\newblock {\em Journal of The Royal Society Interface}, 11(101):20140873, 2014.

\bibitem[RHBK15]{RHBK15}
Jan Reininghaus, Stefan Huber, Ulrich Bauer, and Roland Kwitt.
\newblock A stable multi-scale kernel for topological machine learning.
\newblock In {\em In Proceedings of the IEEE Conference on Computer Vision and
  Pattern Recognition(CVPR)}, pages 4741--4748, 2015.

\bibitem[RR07]{RR07}
Ali Rahimi and Benjamin Recht.
\newblock Random features for large-scale kernel machines.
\newblock In {\em Advances in neural information processing systems}, pages
  1177--1184, 2007.

\bibitem[RT16]{RT16}
Vanessa Robins and Katharine Turner.
\newblock Principal component analysis of persistent homology rank functions
  with case studies of spatial point patterns, sphere packing and colloids.
\newblock {\em Physica D: Nonlinear Phenomena}, 334:99--117, 2016.

\bibitem[SFG13]{SFG13}
Le~Song, Kenji Fukumizu, and Arthur Gretton.
\newblock Kernel embeddings of conditional distributions: A unified kernel
  framework for nonparametric inference in graphical models.
\newblock {\em IEEE Signal Processing Magazine}, 30(4):98--111, 2013.

\bibitem[SFL11]{SFL11}
Bharath~K. Sriperumbudur, Kenji Fukumizu, and Gert R.~G. Lanckriet.
\newblock Universality, characteristic kernels and rkhs embedding of measures.
\newblock {\em The Journal of Machine Learning Research}, 12:2389--2410, 2011.

\bibitem[SGSS07]{SGSS07}
Alex Smola, Arthur Gretton, Le~Song, and Bernhard Sch{\"o}lkopf.
\newblock A hilbert space embedding for distributions.
\newblock In {\em In Algorithmic Learning Theory: 18th International
  Conference}, pages 13--31. Springer, 2007.

\bibitem[SMI{\etalchar{+}}08]{SMISCR08}
Gurjeet Singh, Facundo Memoli, Tigran Ishkhanov, Guillermo Sapiro, Gunnar
  Carlsson, and Dario~L Ringach.
\newblock Topological analysis of population activity in visual cortex.
\newblock {\em Journal of vision}, 8(8):11, 2008.

\bibitem[STR{\etalchar{+}}17]{STRFH17}
Mohammad Saadatfar, Hiroshi Takeuchi, Vanessa Robins, Nicolas Francois, and
  Yasuaki Hiraoka.
\newblock Pore configuration landscape of granular crystallization.
\newblock {\em Nature Communications}, 8:15082 EP, 2017.

\bibitem[TTD00]{TTD00}
Salvatore Torquato, Thomas~M Truskett, and Pablo~G Debenedetti.
\newblock Is random close packing of spheres well defined?
\newblock {\em Physical review letters}, 84(10):2064, 2000.

\bibitem[XW14]{XW14}
Kelin Xia and Guo-Wei Wei.
\newblock Persistent homology analysis of protein structure, flexibility, and
  folding.
\newblock {\em International journal for numerical methods in biomedical
  engineering}, 30(8):814--844, 2014.

\bibitem[ZC05]{ZC05}
Afra Zomorodian and Gunnar Carlsson.
\newblock Computing persistent homology.
\newblock {\em Discrete \& Computational Geometry}, 33(2):249--274, 2005.

\end{thebibliography}

\end{document}